\documentclass[twoside, onecolumn]{article}
\usepackage[margin=1in]{geometry}

\usepackage{natbib}

\usepackage{url}
\usepackage{authblk}
\usepackage{blindtext}
\usepackage{amsmath,amssymb,latexsym}
\usepackage{amsthm}
\usepackage{graphicx}
\usepackage{caption}
\usepackage{subcaption}
\usepackage[ruled]{algorithm2e}
\usepackage[usenames, dvipsnames]{color}
\usepackage{tikz}
\usetikzlibrary{shapes,arrows}
\tikzstyle{block} = [ rectangle, draw, fill=white, text width=5em, text centered, rounded corners, minimum height=4em ]
\tikzstyle{line} = [ draw, -latex' ]

\def\R{\mathbb{R}}
\def\one{\mathbf{1}}
\def\E{\mathbb{E}}
\def\varep{\varepsilon}
\def\Hcal{\mathcal{H}}
\def\Acal{\mathcal{A}}

\def\Hcal{\mathcal{H}}
\def\Ncal{\mathcal{N}}

\def\bK{\mathbf{K}}

\def\Score{\widehat{\theta}}
\DeclareMathOperator*{\argmin}{argmin}

\DeclareMathOperator*{\maximize}{maximize}
\DeclareMathOperator*{\minimize}{minimize}
\DeclareMathOperator*{\sign}{sign}

\newtheorem{lemma}{Lemma}
\newtheorem{theorem}{Theorem}

\newtheorem{definition}{Definition}
\newtheorem{assumption}{Assumption}
\newtheorem{remark}{Remark}

\begin{document}

\title{Sparse Feature Selection in Kernel Discriminant Analysis via Optimal Scoring}

\author[]{Alexander F. Lapanowski}
\author[]{Irina Gaynanova}
\affil[]{Texas A\&M University}
\affil[]{\{alapanow, irinag\}@stat.tamu.edu}

\date{}
\maketitle

\begin{abstract}
 We consider the two-group classification problem and propose a kernel classifier based on the optimal scoring framework. Unlike previous approaches, we provide theoretical guarantees on the expected risk consistency of the method. We also allow for feature selection by imposing structured sparsity using weighted kernels. We propose fully-automated methods for selection of all tuning parameters, and in particular adapt kernel shrinkage ideas for ridge parameter selection. Numerical studies demonstrate the superior classification performance of the proposed approach compared to existing nonparametric classifiers. %We also propose automatic methods for ridge parameter selection and guassian kernel parameter selection.  
\end{abstract}

\section{Introduction}

%Let $X\in \mathbb{R}^{n\times p}$ be the data matrix of $n$ data points with $p$ features, where each data point belongs to one of 2 groups.  
Linear Discriminant Analysis (LDA) is a popular linear classification rule \cite[Section 4.3]{hastie_elements_2009}, but it has two limitations. First, it will underfit the data when the best decision boundary is nonlinear. 
Secondly, LDA uses all $p$ features even though not all may contribute to class separation.  Including such ``noise'' features into the classification rule can harm classification performance.

To account for non-linearity, several authors consider kernel discriminant analysis \cite{baudat2000generalized,mika_fisher_1999, roth2000nonlinear, bernhard_scholkopf_learning_2002}. While the methods have good empirical performance, to our knowledge there is a lack of theoretical guarantees on the risk of the learned classifiers. Recently, \cite{diethe2009matching, kim2006robust, lanckriet2002robust} provided such guarantees, however under modified classification criterion with respect to worst-case training data realization. At the same time, none of the above methods perform feature selection, and as such will overfit in the presence of ``noise" features.

%The majority of kernel theory assumes a convex loss function.  An additional challenge with kernel FDA is incorporating sparse feature selection, as the method developed in \cite{allen_automatic_2013} assumes a convex loss function as well.

On the other hand, several sparse generalizations of LDA have been proposed \cite{Cai:2011dm,Clemmensen:2011kr,Gaynanova:2016wk}, however the methods still result in linear classification boundaries.

This paper addresses the gap between kernel and sparse LDA methods by using an optimal scoring framework~\cite{Hastie:1994cx} to construct a kernel-based classifier. Unlike previous approaches, we provide theoretical guarantees on the risk consistency of the proposed kernel optimal scoring. We also allow the method to perform feature selection by adapting the weighted kernel idea from \cite{allen_automatic_2013}. To avoid computational costs associated with selecting multiple tuning parameters, we develop a new  Stabilization method for ridge parameter selection. The method is based on the shrinkage ideas from \cite{lancewicki_regularization_2017} for stabilization of kernel matrices. Our empirical results indicate that the Stabilization method leads to better error rates than generalized cross-validation (GCV) \cite{Craven1978,  golub1979generalized,xiang1996generalized}, and we believe this method of parameter selection could be of independent interest.

%Kernel classifiers often require the selection of several parameters, but \cite{allen_automatic_2013} does not provide guidance for doing so. This paper provides fully-automatic selection methods for gaussian kernel, ridge, and sparsity parameters which avoids cross-validation over all three parameters. This is done by a new automatic ridge parameter selection technique based on \cite{Lancewicki:2017tz}, which could be of independent interest.

In summary, this work makes the following contributions: (i) we develop a kernel LDA method based on optimal scoring framework; (ii) we provide theoretical results on the risk consistency of the proposed classifier; (iii) we use weighted kernels to implement feature selection within kernel LDA; and (iv) we propose a new stabilization method for ridge parameter selection.

\subsection{Related Work}

In this section we draw connections between our work and existing literature on kernelized optimal scoring as well as sparse feature selection within kernels.

To our knowledge, the kernelized version of the optimal scoring problem has not been considered in the literature except for~\cite{roth2000nonlinear}. Unlike \cite{roth2000nonlinear}, we fix the scores and provide theoretical guarantees for the method. Another major distinction of our method is the feature selection which is achieved by weighting the kernel and adding a sparsity penalty to the weights. 

%\textbf{\textit{Alex, is my assumption correct that you are stating exactly the same references here as used in Allen's paper and refering to them exactly? This is not ok. I will try to correc this as I see git, but generally you can not do lireature review from someone else's paper, you should do it yourself and draw your own conclusions with very few exceptions to that}}
%{\color{red}Yes, that is what I did. I thought her comment was useful in understanding why previous work was different. I won't draw directly from a literature search again. Thank you.}

 %Weighting features within kernels is considered in \cite{cao2007feature, li2006lasso, grandvalet2003adaptive, gilad2004margin, argyriou2006dc, weston2001feature}.  \cite{allen_automatic_2013} states ``Most of these, however, do not directly optimize the original regression or classification problem, but instead seek to find a good set of weights on the features for later use...''
 
 Weighted kernels with sparse weights have been considered in \cite{allen_automatic_2013, chen2017double} in the context of kernel regression and kernel support vector machines. The framework can not be applied to the original kernel LDA method \cite{mika_fisher_1999}, however it could be adapted to the proposed kernel optimal scoring problem due to its least squares formulation.
 
 Learning the optimal weight vector can be viewed as a kernel learning problem. While most of the kernel learning literature focuses on finding linear or quadratic combination of predetermined kernels \cite{bach2004multiple, sonnenburg2006large}, learning the weights corresponds to adjusting the feature support of the kernel matrix.  This is also distinctive from the sparse kernel learning literature, where the kernel is assumed to be additive with respect to the features \cite{bach2008consistency, sun2015learning}. Our framework does not impose additivity, thus enabling interactions between the features.

\subsection{Notation} 
%We use the following notation throughout the paper. 
For a vector $v\in \R^p$, let $\|v\|_{2}:=\sqrt{\sum_{i=1}^{p}|v_i|^2}$ be the Euclidean norm, $\|v\|_{1}:=\sum_{i=1}^{p}|v_i|$ be the $\ell^1$ norm, and $\|v\|_\infty:=\max|v_i|$ be the $\ell^\infty$ norm. Let $\left<x, x'\right>:=\sum_{i=1}^{p}x_{i}x_{i}'$ be the Euclidean inner product in $\mathbb{R}^p$.
For a matrix $M\in \mathbb{R}^{n\times k}$, let $M_{i,j}$ denote the $(i,j)$ element of $M$. Let $\|M\|_{\text{op}}:=\sup_{\|x\|_2=1} \|Mx\|_{2}$ be the operator norm, and let $\|M\|_{F}:= \sqrt{\sum_{i=1}^{n}\sum_{j=1}^{k} |M_{i,j}|^{2}}$ be the Frobenius norm. Let $I$ be the $n\times n$ identity matrix.  Let $\mathbf{1}\in \mathbb{R}^{n}$ be the vector of all $1$s, and let $C= I-n^{-1} \mathbf{1} \mathbf{1}^{\top}$ be the centering matrix.

\section{Kernel Optimal Scoring}\label{sec:rKOS}
\subsection{Linear Discriminant Analysis and Optimal Scoring }\label{sec:LDA}

%In this section we review the connection between LDA and optimal scoring in the case of binary classification. 

Let $\{(x_i, y_i)\}_{i=1}^{n}$ be independent pairs, where $x_i\in \R^p$ is the vector of features, and $y_i\in \R^2$ is the indicator of class membership such that $y_{ik} = 1$ if $i$th sample belongs to class $k$, $i\in C_k$, and $y_{ik} = 0$ otherwise. Let $n_1$ and $n_2$ denote the number of samples in each respective class so that $n = n_1 + n_2$. Let $X\in \mathbb{R}^{n\times p}$ and $Y\in\mathbb{R}^{n\times 2}$ denote the corresponding feature and indicator matrices, and without loss of generality let $X$ be column-centered.

The optimal scoring problem \cite{Hastie:1994cx} finds the discriminant vector $\beta \in \R^p$ and the scores vector $\theta \in \R^2$ by solving
\begin{align}\label{eq:OptScoring}
\begin{split}
&\minimize_{\theta, \beta}\|Y\theta-X\beta\|_{2}^{2}\\
&\text{subject to }n^{-1}\theta^{\top}Y^{\top}Y\theta=1,\,\theta^{\top}Y^{\top}Y \mathbf{1}=0.
\end{split}
\end{align}

Since the solution vector of scores has explicit form up to a sign, $\widehat \theta = (\sqrt{n_2/n_1}\ -\sqrt{n_1/n_2})^{\top}$, ~\eqref{eq:OptScoring} is equivalent to the linear regression problem
\begin{equation}\label{eq:optReg}
\minimize_{\beta}\|Y\widehat \theta - X\beta\|_2^2.
\end{equation}
The solution $\widehat \beta$ corresponds to the discriminant vector in LDA up to scaling \cite[Section 3.4]{hastie_penalized_1995}. Thus, linear discriminant analysis can be reduced to finding the solution to problem~\eqref{eq:optReg}.

%Let $G_{k}$ be group $k$ with size $n_k$ for $k=1,2$. Let $\bar x_k=\frac{1}{n_k}\sum_{i\in G_{k}}x_{i}$ be the mean of group $G_k$, and let $\bar x=\frac{1}{n}\sum_{i=1}^{n}x_i$ be the overall mean. The between group covariance matrix $B$ and within-group covariance matrix $W$ are defined as
%\begin{align}\label{eq:CovMatrices}
%\begin{split}
%&B:=\frac{1}{2}\sum_{k=1}^{2}n_i(\bar x_{k}-\bar x)(\bar x_{k}-\bar x)^{\top}\\
%&W:=\frac{1}{n}\sum_{k=1}^{2}\sum_{i\in G_k}(x_{i}-\bar x_{k})(x_{i}-\bar x_{k})^{\top}.
%\end{split}
%\end{align} 
%LDA seeks projections which maximally 
%separate groups of 
%data with respect to the within-group variability \cite[Section 4.3]{hastie_elements_2009} by solving
%\begin{equation}\label{eq:LDA}
%\maximize_{v\neq 0}\,\frac{v^{\top}Bv}{v^{\top}Wv}.
%\end{equation}

\subsection{Reproducing Kernel Hilbert Spaces}\label{sec:ReproducingKernelHilbertSpace}

Reproducing Kernel Hilbert Spaces (RKHS) are commonly used in creating non-linear classifiers. The data is mapped into a RKHS $\mathcal{H}$ via $\Phi:\mathbb{R}^{p}\to \mathcal{H}$ with an accompanying kernel
$k:\mathbb{R}^{p}\times \mathbb{R}^{p}\to\mathbb{R}$ such that $\left<\Phi(x),\Phi(x')\right>_{\mathcal{H}}=k(x,x')$ for any $x, x' \in \mathbb{R}^{p}$. We let $\|\cdot\|_{\mathcal{H}}$ be the norm induced by the inner product $\left<\cdot\,,\,\cdot\right>_\mathcal{H}$. By the \emph{reproducing property} of $\mathcal{H}$: $\left<\Phi(x), f\right>_{\Hcal}=f(x)$ for all $x\in \mathbb{R}^{p}$ and $f\in \mathcal{H}$. Thus, any classifier that relies on the training data only through the inner products can be \emph{kernelized} by substituting kernel evaluations in place of inner products. This effectively creates a classifier in $\mathcal{H}$ rather than in $\R^p$. 

Some commonly-used kernels are the gaussian kernel $k(x,x')=\exp(-\sigma^{-2} \|x-x'\|_{2}^2)$ with $\sigma >0$, the polynomial kernel $k(x,x')= (1+\left<x, x'\right>)^{d}$ with $d$ a positive integer, and the sigmoid kernel $k(x,x')=\tanh(c\left<x, x'\right>+t)$ with $c>0$, $t\geq 0$.
%However, different kernels are used for different problems, and often there are specific classes of kernels which are designed to incorporate prior knowledge of a given problem.
We refer the reader to \cite[Chapter 13]{bernhard_scholkopf_learning_2002} for a review on kernel construction and selection. We let $\mathbf{K}\in \mathbb{R}^{n\times n}$ denote the kernel matrix $\mathbf{K}_{i,j}:=k(x_i, x_j)$ based on observed feature vectors $\{x_i\}_{i=1}^n$.

% \emph{The Representer Theorem} and its variants
%state that solutions to many kernelized algorithms lie in 
%the finite-dimensional span of the 
%mapped data \cite{kimeldorf_correspondence_1970}, \cite{scholkopf_generalized_2001}. Thus, optimization problems in the Hilbert Space 
%$\mathcal{H}$ may be solved in the finite-dimensional coefficient space of the mapped data. The algorithm proposed in this paper, kernel optimal scoring (KOS), uses the classical representer theorem \cite{kimeldorf_correspondence_1970}.

%hroughout, we use $\mathcal{H}$ for Reproducing Kernel Hilbert Space with kernel $k$ and map $\Phi$, and let $\|\cdot\|_{\mathcal{H}}$ be the norm induced by the inner product $\left<\cdot\,,\,\cdot\right>_\mathcal{H}$. 

\subsection{Kernel Optimal Scoring}\label{sec:RegKernOptScore}

We derive the kernelized formulation of the optimal scoring problem~\eqref{eq:optReg}. Let $f$ be the discriminant function in $\Hcal$ with corresponding map $\Phi$ and kernel $k$. We substitute each inner product $x_i^{\top}\beta = \langle x_i, \beta \rangle$ with inner product in $\Hcal$, $\langle \Phi(x_i) - \overline{\Phi}, f \rangle_{\Hcal}$, where we apply centering to $\Phi(x_i)$ via $\overline{\Phi}:=n^{-1}\sum_{i=1}^{n}\Phi(x_i)$ to take into account column-centering of $X$. The corresponding optimal scoring problem in $\Hcal$ takes the form

%To kernelize optimal scoring, \eqref{eq:OptimalScoringNoScore} must be expressed in terms of 
%inner products. We have
%\begin{align*}
%&\minimize_{B\in\mathbb{R}^{p\times 1}}\,\frac{1}{n}\|Y\theta-XB \|_{2}^{2}\\
%&=\minimize_{B\in\mathbb{R}^{p\times 1}}\,\frac{1}{n}\bigg\|Y\theta-\begin{pmatrix}
%x^{\top}_{1}B\\
%\vdots\\
%x_{n}^{\top}B
%\end{pmatrix}
%\bigg\|_{2}^{2}.
%\end{align*}
%Let $\Phi:\mathbb{R}^{p}\to\mathcal{H}$ be the accompanying map to a given kernel $k:\mathbb{R}^{p}\times \mathbb{R}^{p}\to \mathbb{R}.$
%The mapped data is centered $\Phi(x_i)\mapsto \Phi(x_i)-\overline{\Phi}$, where $\overline{\Phi}:=\frac{1}{n}\sum_{j=1}^{n}\Phi(x_j)$.
%The discriminant vector in $\mathcal{H}$ is denoted by $f$. Kernelization substitutes $\left<\Phi(x_i)-\overline{\Phi},f\right>_{\mathcal{H}}$ in for $x_{i}^{\top}B$ to yield 

\begin{equation*}
\minimize_{ f\in \mathcal{H}}\,\bigg\|Y\widehat\theta -\begin{pmatrix}
\left<\Phi(x_1)-\overline{\Phi}\,,\, f\right>_{\mathcal{H}}\\
\vdots\\
\left<\Phi(x_n)-\overline{\Phi}\,,\, f\right>_{\mathcal{H}}
\end{pmatrix}\bigg\|_{2}^{2}.
\end{equation*}
By the Representer Theorem \cite{kimeldorf_correspondence_1970}, the minimizing $\widehat f$ lies in the finite-dimensional span of the centered data, that is it is sufficient to consider minimization over $f=\sum_{i=1}^{n}\alpha_{i}[\Phi(x_i)-\overline{\Phi}]$ for some $\alpha_{i}\in \mathbb{R}$. Combining the Representer Theorem with kernel representation of inner-products in $\Hcal$
leads to the equivalent coefficient space formulation of the kernel optimal scoring problem:
\begin{equation}\label{eq:UnRegKernOptScore}
\minimize_{\alpha\in \mathbb{R}^{n}}\,\|Y\widehat \theta-C\mathbf{K}C\alpha\|_{2}^{2}.
\end{equation}
%We call~\eqref{eq:UnRegKernOptScore} Kernel Optimal Scoring problem, or KOS.

%{\textit{ \textbf{Alex, I forgot, do we use the equivalence anywhere? If we do, do we cite this remark, or do it separately? I am trying to decide whether we need this remark at all, I am keeping it for now and come back to it}}

%{\color{red} Alex- we do not use it. I included it as a way to motivate the modified empirical loss function. We should delete it.}

%\begin{remark}
%Rather than centering the mapped data, an equivalent approach is to use uncentered data but add an intercept term
%\begin{align*}
%&\minimize_{\beta\in \mathbb{R}, f\in\mathcal{H}} \frac{1}{n}\bigg\|Y\Score -\beta\mathbf{1}-\begin{pmatrix}
%\left<\Phi(x_1),f\right>_{\mathcal{H}}\\
%\vdots\\\left<\Phi(x_n),f\right>_{\mathcal{H}}
%\end{pmatrix}\bigg\|_{2}^{2}\\
%&=\minimize_{\beta\in \mathbb{R},\, \alpha\in %\mathbb{R}^{n}}\frac{1}{n}\|Y\Score-\beta\mathbf{1}-\mathbf{K}\alpha\|_{2}^{2}.
%\end{align*}
%We do not pursue this direction here.
%\end{remark}

%\subsection{Regularized Kernel Optimal Scoring}\label{sec:RegKernOptScore}

Kernel methods may over-fit the training data without further restriction on the set of functions $f\in \Hcal$, \cite{hastie_elements_2009, bernhard_scholkopf_learning_2002,RKHS_tutorial2012}. A common approach is to restrict the norm $\|f\|_{\Hcal}^2 = \alpha^{\top}CKC\alpha$, and we add a ridge penalty to the objective function~\eqref{eq:UnRegKernOptScore}
\begin{equation}\label{eq:RegKernOptScore}
\minimize_{\alpha\in \mathbb{R}^n}\bigg\{\frac{1}{n}\|Y\Score-C\mathbf{K}C\alpha\|_{2}^{2}+\gamma \alpha^{\top}C\mathbf{K}C\alpha\bigg\},
\end{equation}
where $\gamma > 0$ controls the level of regularization. For numerical stability, we also add $\varepsilon I$ with small $\varepsilon > 0$ to the ridge penalty so that $CKC$ is replaced with $CKC + \varepsilon I$. A similar adjustment is used in \cite{mika_fisher_1999,roth2000nonlinear}. We fix $\varepsilon=10^{-5}$ throughout the manuscript. The problem has a closed-form solution leading to 
\begin{equation}\label{eq:alpha}
\widehat{\alpha}= \{(C\mathbf{K}C)^2+n\gamma(C\mathbf{K}C+\varepsilon I)\}^{-1}C\mathbf{K}CY\widehat \theta.
\end{equation}
 We call~\eqref{eq:RegKernOptScore} the kernel optimal scoring problem or KOS.

%This gives the regularized kernel optimal scoring problem:
%\begin{equation}\label{eq:RidgeKOS}
%\minimize_{\alpha\in \mathbb{R}^n}\bigg\{\frac{1}{n}\|Y\Score-C\mathbf{K}C\alpha\|_{2}^{2}+\gamma\,\alpha^{\top}(C\mathbf{K}C+\varepsilon I)\alpha\bigg\}.
%\end{equation}
%The closed-form solution is

\subsection{Classification of a New Data Point}\label{sec:Projection}
In this section we describe how to use KOS for classification. Let $\widehat \alpha$ be as in \eqref{eq:alpha}, and let $\widehat f = \sum_{i=1}^n \widehat \alpha_i[\Phi(x_i) - \overline \Phi]$. Given a new data point $x\in \R^{p}$, let
%$K(X,x)$ denote the column vector 
$$
K(X,x)=
\begin{pmatrix}
k(x_1, x)&
\cdots&
k(x_n,x)\end{pmatrix}^{\top}.
$$
We define the projected value $P(x)$ as the inner-product between $x$ mapped and centered in $\Hcal$ and $\widehat f$ so that $P(x)$ is equal to
\begin{equation}\label{eq:Projection}
\left<\Phi(x)-\overline{\Phi},\,
\widehat{f}\right>_{\mathcal{H}}=(K(X,x)^{\top}-n^{-1}\mathbf{1}^{\top}\mathbf{K})C\widehat\alpha.
\end{equation}
The derivation of~\eqref{eq:Projection} is in the Supplement.

KOS classifies $x\in \mathbb{R}^{p}$ using nearest centroids classification on the projected values. Specifically, let $\mu_k=\frac{1}{n_k}\sum_{i\in G_k}P(x_x)$ be the mean projected values of group 
$k$ (projected centroid). We classify $x\in \R^p$ according to the minimal distance to projected centroids
$$
\argmin_{k=1,2} |P(x)-\mu_{k}|.
$$

\section{Error Bounds for Kernel Optimal Scoring}\label{sec:Theory}

Problem~\eqref{eq:RegKernOptScore} can be viewed as a regularized empirical risk minimization problem
\begin{equation}\label{eq:penalty}
\widehat f = \argmin_{f\in \mathcal{H}}\left\{R_{\text{emp}}(f) + \gamma \|f\|^2_{\mathcal{H}}\right\},
\end{equation}
where for a fixed $f\in \mathcal{H}$
\begin{equation}\label{eq:emprisk}
R_{\text{emp}}(f):=\frac{1}{n}\sum_{i=1}^{n}|y_i^{\top}\Score-\left<\Phi(x_i)-\overline{\Phi},f\right>|^{2}.
\end{equation}
By duality, for every $\gamma \geq 0$ there exists a $\tau \geq 0$ such that 
\begin{equation}\label{eq:constraint}
\widehat f = \argmin_{\|f\|_{\Hcal}\leq \tau}\left\{R_{\text{emp}}(f)\right\}.
\end{equation}
While the relationship between $\gamma$ and $\tau$ is data-dependent, Lemma~3 in the Supplement shows that $\tau \leq C\min(\gamma^{-1},\gamma^{-1/2})$ for some constant $C>0$. For technical clarity, we analyze~\eqref{eq:constraint} throughout.

There are two complications in analyzing the empirical risk in~\eqref{eq:emprisk}: $\Score$ is dependent on all $y_i$ through $n_1$, $n_2$, and $\overline \Phi$ is dependent on all $x_i$. Hence, the error terms $|y_i^{\top}\Score-\left<\Phi(x_i)-\overline{\Phi},f\right>|^{2}$ are dependent. The empirical risk can be equivalently written as
$$
R_{\text{emp}}(f, \beta) = \frac{1}{n}\sum_{i=1}^{n}|y_i^{\top}\Score-\beta -\left<\Phi(x_i),f\right>|^{2},
$$
with the minimizing $\widehat \beta =-\langle{\overline{\Phi}, f\rangle}$ since $\one^{\top}Y\widehat \theta = 0$. %Letting $\widehat \beta = -\langle\overline{\Phi}, f\rangle$,
We therefore introduce a modified empirical risk using population scores $\theta^*$ and an extra intercept parameter $\beta\in \R$. The population scores $\theta^*$ result from substituting $\pi_k$ instead of $n_k/n$ in $\widehat \theta$.
\begin{definition}\label{eq:PopulationScores}
Let $\pi_k = P(i\in C_k)$ be the prior class probabilities, $k=1,2$. The \emph{population scores} are defined as
$\theta^*=(
\sqrt{\pi_2/\pi_1}\ -\sqrt{\pi_1/\pi_2}
)^{\top}$.
\end{definition}
For a fixed $f\in\Hcal$ and $\beta\in\R$, the modified empirical risk is
$$
\widetilde{R}_{\text{emp}}(f,\beta)=\frac{1}{n}\sum_{i=1}^{n}|y_i^{\top}\theta^*-\beta-\left<\Phi(x_i),f\right>|^{2}.
$$
Unlike the empirical risk, the modified empirical risk is the average of iid terms. For a fixed $f\in \Hcal$ and $\beta \in \R$, the corresponding expected risk is
%\begin{definition}[Expected Risk]
%For a fixed $f\in \mathcal{H}$ and $\beta\in \mathbb{R}$, the %\emph{expected} risk of $(f, \beta)$ is
$$
R(f,\beta):=\mathbb{E}_{(x,y)}|y^{\top}\theta^*-\beta-\left<\Phi(x),f\right>|^{2}.
$$
%\end{definition}

Let $\widehat f$ be as in~\eqref{eq:constraint} and let $\widehat \beta = -\langle \overline{\Phi}, \widehat f\rangle$. We next derive probabilistic bounds on the expected risk of $\widehat f$. Throughout, we use the following assumptions.

\begin{assumption}\label{assump:ScoreBound} Let $\pi_{\max} = \max(\pi_1, \pi_2)$, $\pi_{\min} = \min(\pi_1, \pi_2)$. 
There exists a constant $C>0$ such that $\|\theta^*\|_\infty = \sqrt{\pi_{\max}/\pi_{\min}}\leq C$.
\end{assumption}
This assumption implies that the prior group probabilities are not degenerate, that is $\pi_1 \asymp \pi_2$. 

\begin{assumption}\label{assump:KernBound}
There exists a constant $\kappa>0$ such that $\|\Phi(x)\|_{\mathcal{H}}\leq \kappa$ for all $x\in \mathbb{R}^{p}$. Equivalently, $
\sup_{x\in\mathbb{R}^{p}} k(x,x)\leq \kappa^{2}.$
\end{assumption}

\begin{assumption}\label{assump:Separable}
The RKHS $\mathcal{H}$ is separable.
\end{assumption}

\begin{remark}
The gaussian kernel satisfies Assumption~\ref{assump:KernBound} with $\kappa=1$ and satisfies Assumption~\ref{assump:Separable} by Theorem~7 in \cite{hein2004kernels}. 
\end{remark}

Using~\eqref{eq:constraint}, we define the set of admissible functions $f$ as $\Hcal_{\tau}:=\{f\in\Hcal: \|f\|_{\Hcal}\leq \tau\}$, and the set of admissible intercepts $\beta$ as $I_{\tau}:=\{\beta \in \R: |\beta|\leq \|\theta^*\|_{\infty} + \kappa \tau\}$.

%The next two assumptions are needed to formally specify the sets of admissible functions and intercepts.
%\begin{definition}\label{def:AdmissibleFunctions}
%Let $\kappa>0$ be the kernel bound of Assumption \ref{assump:KernBound}.
%The \emph{set of admissible functions} is defined to be 
%\[
%\mathcal{H}_\gamma:=\bigg\{f\in \mathcal{H}: \|f\|_\mathcal{H}\leq 
%\min\bigg[\frac{2\kappa}{\gamma}, %\frac{1}{\sqrt{\gamma}}\bigg]\bigg\}.
%\] The set of \emph{admissible %intercepts} is defined to be
%\[
%I_\gamma:=\bigg\{\beta\in \mathbb{R}: %|\beta|\leq %\|\theta^*\|_\infty+\kappa\min\bigg[ \frac{2\kappa}{\gamma}, \frac{1}{\sqrt{\gamma}}\bigg]\bigg\}.
%\]
%\end{definition}

\begin{remark}
The intercept $\widehat \beta\in I_{\tau}$ by Assumption~\ref{assump:KernBound}. The extra term $\|\theta^*\|_{\infty}$ comes from minimizing the modified empirical risk.
\end{remark}

Let 
\begin{equation}\label{eq:tilde}
(\widetilde{f} , \widetilde{\beta}):=\argmin_{f\in \mathcal{H}_\tau\,,\, \beta\in I_\tau} \widetilde{R}_{\text{emp}}(f, \beta).
\end{equation}
be the minimizers of the modified empirical risk over the set of admissible functions and intercepts, and let 
\begin{equation}\label{eq:star}
(f^\ast, \beta^\ast)=\argmin_{f\in \mathcal{H}_\tau \,,\, \beta\in I_\tau} R(f, \beta)
\end{equation}
be the minimizers of the expected risk over the set of admissible functions and intercepts. Our proofs rely on characterizing (i) the difference between ~\eqref{eq:constraint} and~\eqref{eq:tilde}, and (ii) the difference between~\eqref{eq:tilde} and~\eqref{eq:star}. The detailed proofs are in the Supplement, and below we state the main results.

%\begin{remark}\label{rem:AdmisisbleIntercept}
%For a fixed $f\in \mathcal{H}$, the intercept $\beta$ which minimizes \eqref{eq:ModEmpRisk} is
%$\widetilde{\beta}:=\overline{y\theta^*}-\left< \overline{\Phi}, f\right>$, where $\overline{y\theta^*}:= \frac{1}{n}\sum_{i=1}^{n} y_{i}\theta^{0}$. If $f$ satisfies one of the bounds of Lemma \ref{lem:Bounds}, the triangle and Cauchy-Schwarz inequalities prove that
%$|\widetilde{\beta}|\leq \|\theta^*\|_\infty+\kappa\min(2\kappa/\gamma, 1/\sqrt{\gamma})$.
%\end{remark}

%The estimation error and approximation error are defined as follows:
%\begin{align*}
%&R(\widehat{f},\widehat{\beta})-\inf_{f\in \mathcal{H}, \beta\in I}R(f,\beta)=\\
%&\underbrace{R(\widehat{f},\widehat{\beta})- R(f^*,\beta^*)}_{\text{Estimation Error}}
%+\underbrace{
% R(f^*,\beta^*)-
%\inf_{f\in \mathcal{H}\,,\, \beta\in \mathbb{R}} %R(f,\beta)}_{\text{Approximation Error}}.
%\end{align*}

%\subsection{Main Results}\label{sec:MainResults}
%This subsection presents probabilistic upper bounds on the risk and empirical risk of the learned classifier $(\widehat f, \widehat \beta)$.

\begin{theorem}\label{thm:EstErrorProbBound}
Under Assumptions~\ref{assump:ScoreBound}--\ref{assump:Separable}, there exist constants $C_1, C_2, C_3>0$ such that
\begin{align*}
\mathbb{P}\Big(R(\widehat f, \widehat \beta)&>R(f^*,\beta^*)+\varepsilon 
\Big)\leq C_{1}\mathcal{N}_{\varepsilon}\exp\Big(-\frac{C_3n \varepsilon^2}{(\|\theta^*\|_\infty+\kappa\tau)^4}\Big),
\end{align*}
where $\mathcal{N}_{\varepsilon}=\{1+2(\|\theta^*\|_\infty+\kappa\tau)/\varepsilon\}\exp( C_2\tau^2 \varepsilon^{-2})$.
\end{theorem}

\begin{theorem}\label{thm:EmpRiskProbBound}
Under Assumptions~\ref{assump:ScoreBound}--\ref{assump:Separable}, there exist constants $C_1, C_2, C_3>0$ such that
\begin{align*}
\mathbb{P}\bigg(R(\widehat f, \widehat \beta) &> R_{\text{emp}}(\widehat f) +\varepsilon \bigg)\leq  C_1 \mathcal{N}_{\varepsilon}\exp\Big(-\frac{C_3n \varepsilon^2}{(\|\theta^*\|_\infty+\kappa\tau)^4}\Big),
\end{align*}
where $\mathcal{N}_{\varepsilon}=\{1+2(\|\theta^*\|_\infty+\kappa\tau)/\varepsilon \}\exp( C_2\tau^2 \varepsilon^{-2})$.
\end{theorem}

Theorem \ref{thm:EstErrorProbBound} bounds the expected risk of $\widehat{f}$ compared to the best in-class expected risk, whereas Theorem \ref{thm:EmpRiskProbBound} bounds it in terms of the empirical risk of $\widehat{f}$. 

\section{Sparse Kernel Optimal Scoring}\label{sec:SparseKOS}
The regularized KOS problem~\eqref{eq:RegKernOptScore} performs no feature selection. All $p$ features are used in construction of $\widehat{f}$ and the subsequent classification rule. In many applications, however, it is reasonable to expect that not all features contribute to class separation. Including such noisy features in the discriminant rule can lead to poor classification performance. %See \cite[Section 4.1.2]{allen_automatic_2013} for a particular case of this phenomenon. As another example, 
Figure \ref{fig:TrainTestScatterplot} shows an example based on simulated data with four features. Only the first two features contribute to class separation, while the third and fourth features are noise.

%\begin{figure}
%\centering
%\begin{subfigure}[!t]{.5\linewidth}
%  \centering
%  \includegraphics[scale=.5]{SimulatedData.pdf}
%\end{subfigure}%
%\begin{subfigure}[t]{.5\linewidth}
%  \includegraphics[scale=.5]{Model1Projections.pdf}
%\end{subfigure}
%\caption{\textbf{Left: } Simulated training and test data with four features, only features $1$ and $2$ contribute to class separation. \textbf{Right: }Comparing the projection values \eqref{eq:Projection} of the test data in Figure \ref{fig:TrainTestScatterplot} with and without sparsity.}
%\end{figure}

\begin{figure}
\centering
\includegraphics[width=8cm]{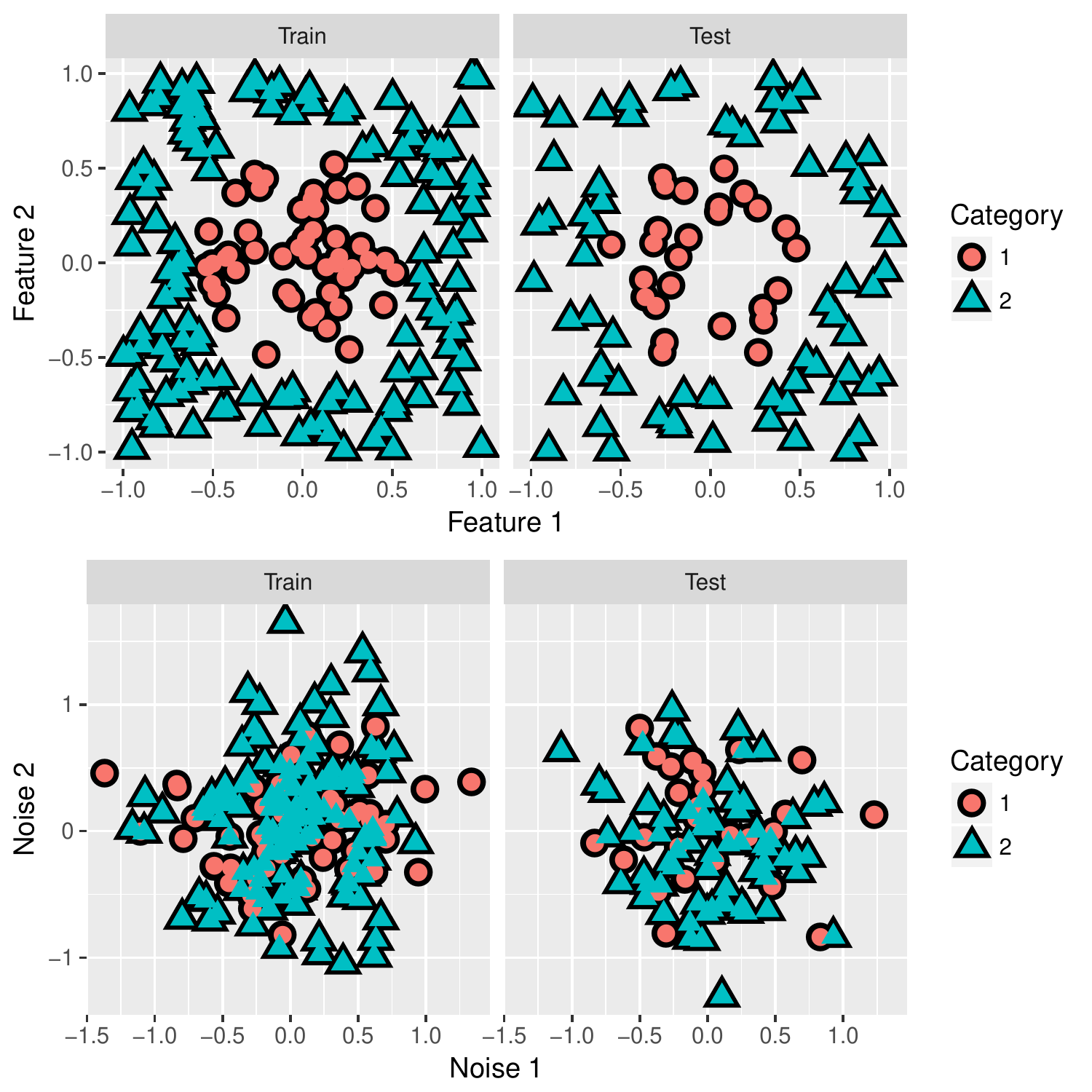}
 \caption{Simulated training and test data with four features, only features $1$ and $2$ contribute to class separation.} 
\label{fig:TrainTestScatterplot}
\end{figure}

 Figure \ref{fig:ProjectionHistograms} shows the projected data values \eqref{eq:Projection} formed by applying KOS to (i) all four features and (ii) only the first two features.
The class separation is perfect based on the two ``true" features, but the projected values overlap with the addition of noisy features, thus illustrating the need for feature selection within KOS.

\begin{figure}[!t]
\centering
\includegraphics[width=8cm]{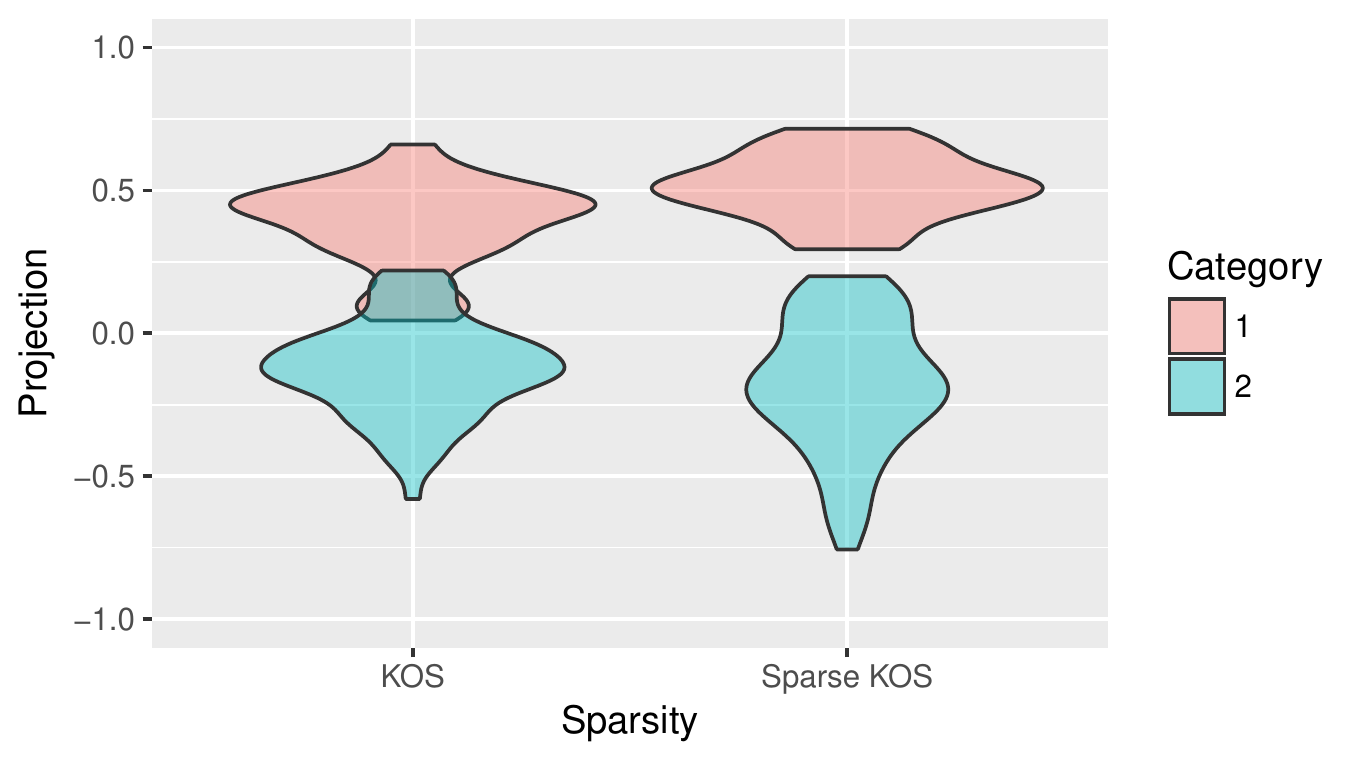}
 \caption{Comparing the projection values \eqref{eq:Projection} of the test data in Figure \ref{fig:TrainTestScatterplot} with and without sparsity.}
\label{fig:ProjectionHistograms}
\end{figure}

To incorporate feature selection, we borrow the ideas from \cite{allen_automatic_2013} and introduce a weight vector $w\in \R^p$, where we restrict each feature as $w_j \in [-1,1]$. The weight vector is used to form the weighted kernel matrix $(\mathbf{K}_{w})_{i,j}=k(wx_i, wx_j)$, where $wx=(w_1 x_1, \dots, w_p x_p)^{\top}$ is the Hadamard product between the weight vector $w$ and observed feature vector $x$. If $w=\one$, $\bK_w = \bK$ from Section~\ref{sec:RegKernOptScore}. Otherwise, $w$ can be used to rescale features with respect to each other, and more importantly perform feature selection. If $w_j = 0$ for some feature $j$, then the kernel matrix $\bK_w$ is formed without the $j$th feature, successfully eliminating that feature from the classification rule. The main difficulty is that the optimal weight vector $w$ is unknown, and therefore has to be learned in addition to learning the discriminant function $f$.

We adjust~\eqref{eq:RegKernOptScore} to perform joint minimization over the coefficient vector $\alpha \in \R^n$ and the weight vector $w\in \R^p$. To encourage feature selection, we add an $\ell_1$-penalty on $w$ as in \cite{allen_automatic_2013} leading to the following minimization problem:
\begin{equation}\label{eq:ObjFun}
\begin{split}
\minimize_{\alpha\in \mathbb{R}^{n},\,w\in \mathbb{R}^{p}}\bigg\{ \frac{1}{n}&\|Y\Score-C\mathbf{K}_wC\alpha\|_{2}^{2}+\lambda\|w\|_{1}+\gamma \alpha^{\top}(C\mathbf{K}_wC+\varepsilon I)\alpha\bigg\}\\
\text{subject to}&\quad-1\leq w_i\leq 1\text{ for }i=1, \dots, p. 
\end{split}
\end{equation}
%Substitute the weighted kernel matrix in to \eqref{eq:RidgeKOS} and add a LASSO penalty on the weight vector $w\in\mathbb{R}^{p}$ to get the objective function
%\begin{equation}\label{eq:ObjectiveFunction}
%\frac{1}{n}\|Y\theta-C\mathbf{K}_wC\alpha\|_{2}^{2}+\lambda\|w\|_{1}+\gamma\alpha^{\top}(C\mathbf{K}_wC+\varepsilon I)\alpha.
%\end{equation}
Here $\lambda \geq 0$ is the tuning parameter that controls the sparsity of the weight vector $w$, with larger values leading to sparser solutions. We call~\eqref{eq:ObjFun} sparse kernel optimal scoring. Given the solution pair $(\widehat w, \widehat \alpha)$, we perform classification as in Section~\ref{sec:Projection} with $\bK_{\widehat w}$ being substituted for $\bK$ and $\widehat w x$ substituted for $x$ in forming the projected values $P(x)$ in~\eqref{eq:Projection}.

\begin{remark} Unlike our restriction $w_k\in [-1,1]$,
\cite{allen_automatic_2013} considers $w_k \in [0,1]$. Both lead to $w_k^2\in [0,1]$, but we found that the latter may force all the weights to zero even when $\lambda = 0$. % and $\gamma$ is sufficiently large.
This behavior is avoided when the weights are allowed to be negative.
\end{remark}

\subsection{Optimization Algorithm}\label{sec:Algorithm}

In this section we describe the optimization algorithm for problem~\eqref{eq:ObjFun} given the fixed values of  $\gamma, \lambda\ge 0$. Methods for parameter selection are presented in Section \ref{sec:ParameterSelection}. We define the objective function in~\eqref{eq:ObjFun} as
\begin{equation}\label{eq:ObjectiveFunction}
\begin{split}
 Obj(w, \alpha)& = \frac{1}{n}\|Y\Score-C\mathbf{K}_wC\alpha\|_{2}^{2}+\lambda\|w\|_{1}+\gamma \alpha^{\top}(C\mathbf{K}_wC+\varepsilon I)\alpha.
 \end{split}
\end{equation} 

There are two challenges in solving~\eqref{eq:ObjFun}: (i) non-convexity of the objective function \eqref{eq:ObjectiveFunction} in $(\alpha, w)$ and (ii) non-convex mapping $w\mapsto \mathbf{K}_w$. \cite{allen_automatic_2013} propose to overcome these challenges by (i) iterative minimization over $\alpha$ 
and $w$ and (ii) linearization of the weighted kernel matrix $\bK_w$ with respect to the current value of the weight vector. We adapt the algorithm from \cite{allen_automatic_2013} to problem~\eqref{eq:ObjFun}.

Given the current value of the weight vector $w$, we form the corresponding weighted kernel matrix $\bK_w$ and update $\alpha$ according to~\eqref{eq:alpha} with $\bK$ substituted with $\bK_w$. Given the current value of the coefficient vector $\alpha$, we update $w$ by linearizing the kernel matrix. Consider the first-order Taylor approximation of 
$\mathbf{K}_{w}$ with respect to $w$ centered at the previous value $w^{(t-1)}$ elementwise: 
%{\color{red} this is a spot we might be able to reduce. Present the linearized kerenel and then state that the linearized Sparse KOS problem is equivalent to minimizing some quadratic form. Remove the definition of $T$, $Q$, and $\beta$. We can put those definitions in the Supplement.}
\begin{align*}
\widetilde{\mathbf{K}}_{w}(x_i,x_j):=\mathbf{K}_{w^{(t-1)}}(x_i, x_j)+\{\nabla_{w}\mathbf{K}_{w^{(t-1)}}(x_i, x_j)\}^{\top}(w-w^{(t-1)}),
\end{align*}
where $\nabla_w\mathbf{K}_{w^{(t-1)}}(x_i, x_j)\in \R^p$ is the gradient of $k(wx_i, wx_j)$ with respect to $w$ evaluated at $w^{(t-1)}$. We substitute $\widetilde{\mathbf{K}}_{w}$ in place of $\mathbf{K}_{w}$ within \eqref{eq:ObjFun}.
Let $T\in \R^{n\times p}$ be
$$
T :=
\begin{pmatrix}
\sum_{\ell=1}^{n}(C\alpha)_{\ell}\nabla_{w}\, \mathbf{K}_{w^{(t-1)}}(x_1, x_\ell)^{\top}\\
\vdots\\
\sum_{\ell=1}^{n}(C\alpha)_{\ell}\nabla_{w}\, \mathbf{K}_{w^{(t-1)}}(x_n, x_\ell)^{\top}
\end{pmatrix}.
$$

For fixed $\alpha$, the minimization problem~\eqref{eq:ObjFun} with respect to $w$ can be written as
\begin{equation}
\begin{split}\label{eq:wlasso}
&\minimize_{w} \bigg\{\frac{1}{2}w^{\top}Qw-\beta^{\top}w+ \frac{\lambda}{2}\|w\|_{1}\bigg\}\\
&\text{subject to }-1\leq w_i\leq 1\text{ for }i=1, \dots, p;
\end{split}
\end{equation}
where 
\begin{equation}\label{eq:QandB}
\begin{split}
&Q=\frac{1}{n}(CT)^{\top}CT \in \R^{p\times p},\\
&\beta=\frac{1}{n}T^{\top}C
[Y\Score-C\mathbf{K}_{w^{(t-1)}}C\alpha+CTw^{(t-1)}]-2^{-1}\gamma T^{\top}C\alpha\in \mathbb{R}^{p}.
\end{split}
\end{equation}

%Section~\ref{sec:UpdateWeights} provides details on kernel linearization and weight vector update, while the full algorithm is presented in Algorithm~\ref{a:algorithm}.

\begin{algorithm}[!t]
 \SetAlgoLined
 \caption{Sparse Kernel Optimal Scoring}\label{a:algorithm}
  \SetKwInOut{Input}{Input}
  \SetKwInOut{Output}{Output}
\DontPrintSemicolon 
  
\Input{$X\in \R^{n\times p}$, $Y\in \R^{n \times 2}$, $\widehat \theta$, $\sigma >0$, $\gamma >0$, $\lambda \geq 0$ , convergence threshold $\varepsilon_{\text{con}}$}
\Output{Discriminant coefficients $\widehat{\alpha}$ and feature weights $\widehat{w}$.}

$t\gets 0$\;
$w^{(0)} \gets \mathbf{1}$\;
$(\mathbf{K}_{w^{(0)}})_{i,j} \gets k(w^0 x_i, w^0 x_j)$, $\bK_{w^{(0)}}\gets \{(\mathbf{K}_{w_0})_{i,j}\}$\;

\Repeat{$\mbox{\textrm{Obj}}(\alpha^{(t)}, w^{(t)})-\mbox{Obj}(\alpha^{(t-1)}, w^{(t-1)})< \varepsilon_{\text{con}}$}{
$t\gets t+1$\;

Update $\alpha^{(t)}$ according to~\eqref{eq:alpha} with $\bK = \bK_{w^{(t-1)}}$\;
 
Update $w^{(t)}$ using coordinate descent with updates according to~\eqref{eq:wupdate}\;

$(\mathbf{K}_{w^{(t)}})_{i,j}\gets k(w^{(t)}x_i, w^{(t)}x_j)$\;
}
\Return{$\widehat{\alpha}=\alpha^{(t)}$, $\widehat{w}=w^{(t)}$}
\end{algorithm}

%\subsection{Update of Weights}\label{sec:UpdateWeights}

%In this section, we describe the update of weight vector using the linearization of kernel matrix as proposed in \cite{allen_automatic_2013}. 

Problem~\eqref{eq:wlasso} is of the same form as the penalized lasso problem \cite[Chapter 5]{hastie2015statistical} with extra convex constraints on $w$. % \cite{boyd2004convex}, and \cite[Chapter 5]{hastie2015statistical}.
%\textbf{... (reference) think of appropriate reference here}
Therefore, we can use the coordinate-descent algorithm to solve~\eqref{eq:wlasso}.

Consider optimizing~\eqref{eq:wlasso} with respect to $w_k$. From the KKT conditions~\cite{boyd2004convex}, the solution must satisfy
\begin{equation}\label{eq:wupdate}
\widehat w_k = \sign(\widetilde w_k)\min(|\widetilde w_k|, 1),
\end{equation}
where
$$
\widetilde{w}_{k}:=\frac{1}{Q_{kk}}S_{\lambda/2}\bigg(\beta_k-\sum_{i\neq k}Q_{ki}w_{i} \bigg),
$$
and $S_{\lambda/2}(x):=\text{sign}(x) \max\{|x|-\lambda/2,\, 0\}$ is the soft-thresholding function. The coordinate-descent algorithm proceeds by applying the update~\eqref{eq:wupdate} on each feature $k$ until convergence.

The full algorithm for~\eqref{eq:ObjFun} is summarized as Algorithm~\ref{a:algorithm}. While the update of $w$ is based on approximation of objective function~\eqref{eq:ObjectiveFunction}, in our experience the objective function is always decreasing at each iteration. In case of convergence issues, one can use a line search along a descent direction of $w$ \cite{allen_automatic_2013}. We refer to \cite{allen_automatic_2013} for further discussion of algorithmic convergence.

\section{Parameter Selection}\label{sec:ParameterSelection}
This section describes the selection of the kernel parameter (tailored to the gaussian kernel parameter $\sigma^2$), ridge parameter $\gamma$, and sparsity parameter $\lambda$.

\subsection{Gaussian Kernel Parameter Selection}\label{sec:KernParameter}

%{\textbf{\textit{Alex, are the methods below specifically for selecting Gaussian kernel parameter or any kernel parameter? The approach you proposing will only work for Gaussian kernel, correct?}}}

%{\color{red} Alex- It could conceivably be used in any context where the kernel parameter scales distances or radius uniformly in all directions. The selection method is guided more by geometric intuition of how the groups of data are spaced in $\mathbb{R}^{p}$ rather than anything specifically gaussian. }

%\cite{li2010automatic} proposes a method based on maximizing the kernel evaluation between points in the same class while also trying to minize the kernel evaluation of points in seperate classes. 
%{\textbf{\textit{I did look into the paper by \cite{li2010automatic}. The approach they propose is actually not restrictive to SVM, and could be used in our framework. Given the time and that the work is already done, we definitely not going to explore it now, but please be more careful in reading the papers in the future, and finding appropriate references to support the statements. \cite{li2010automatic} is an appropriate reference for the new method for kernel parameter selection, it is not an appropriate reference for how common is k-fold cross-validation. Hope this helps in the future.}}}

We propose to use 5-fold cross-validation to minimize the error rate. To reduce computational cost, we only consider five tuning parameters based on the $\{.05, .1, .2, .3,.5\}$ quantiles of the set of squared distances between the classes
$$
\{\|x_{i_1}-x_{i_2}\|_2^2\,:x_{i_1} \in C_1,\,x_{i_2}\in C_2\}.
$$
This approach is similar to the one used in the R package \textsf{kernlab} \cite{karatzoglou2004kernlab}, which takes values between $.1$ and $.9$ quantiles of the distance statistic $\|x-x'\|_2$ between distinct data points taken from a random subset of the full data. \cite{caputo2002appearance} and \cite{karatzoglou2004kernlab} state that good performance can be achieved with any value of $\sigma$ in this range. Our approach is different in that (i) we select one value based on CV, (ii) only look at the distances between classes, and (iii) only consider lower quantiles.
%There is no conventional method for kernel parameter selection.
%Several methods for selecting kernel parameter have been proposed in the literature. $k$-fold cross-validation is a typical approach \cite{mika_fisher_1999}
% \cite{li2010automatic} notes that $k$-fold cross validation is typically used, as in .
%state that the optimal kernel parameter lies between the $.1$ and $.9$ quantiles of the distance statistic $\|x-x'\|$.
%state that ``Pretty much any value within this interval leads to good performance.'' 
%The method used by \cite{karatzoglou2004kernlab} considers all distances between points without regard to group membership, whereas our method only considers distances of points in separate groups. 
 We find that this yields good predictive accuracy, and we conjecture that the reason is the presence of noise features, which inflate the distance values $\|x_{i_1}-x_{i_2}\|_2$.  This is supported by empirical observation that the quantiles based on the full set of features will exceed the corresponding quantiles based on the reduced set of informative features.

%{\textbf{\textit{Very important, I am not confident we understand correctly what exactly kernlab does. I thought originally the distance test statistics is based on two groups, but now I think it's based on every sample in the data (based on reading the code here \url{https://github.com/cran/kernlab/blob/master/R/sigest.R}. I am also not certain whether the selection is based on squared or non-squared distances, sigma or sigma$^2$, and whether the inverse of quantiles is sometimes taken. I suggest we try to reproduce the numbers from the example of sigest to understand what exactly is happening based on their github code. Please also check that the way I adjusted description below actually matched what you are currently implementing. If not, please adjust so it matched exactly}}}

%{\color{red}Alex- I've looked at the code and read the documentation. I think I understand how they select their $\sigma^2$ parameter, but I don't think it's a good idea. 

%As you mentioned, they take a random subset of their data, and then consider the set of pairwise distances between all non-zero distances in that set. They then consider the $.1$, $.5$, and $.9$ quantiles of that distance set. They return the reciprocal of those quantiles. Even though their formulation of the gaussian kernel is different from ours, the two approaches appear to be nearly equivalent because the take the reciprocal of the distance value. The only difference between our method and their's seems to be that they don't consider distances between groups, and we do.

%How do they select a specific parameter value?}

\subsection{Ridge parameter selection}\label{sec:RidgeSelect}

%{\color{blue} Mental note to myself, shorten and move some to introduction}

Due to the computational expense of cross-validation, we propose an alternative approach for ridge parameter selection based on the shrinkage of kernel matrix. \cite{lancewicki_regularization_2017} proposes to stabilize the kernel matrix via shrinkage towards a target matrix and derives an optimal value for the shrinkage parameter. Following~\cite{lancewicki_regularization_2017}, in KOS we want to stabilize $(C\mathbf{K}_wC)^{2}$ with the target matrix $C\mathbf{K}_w C+\varepsilon I$, and therefore consider
$$
(C\mathbf{K}_wC)^{2} + \gamma (C\mathbf{K}_w C+\varepsilon I)
$$
for $\gamma >0$. Let $t = \gamma/(1+\gamma)$, then the optimal value of $t$ is $\widehat t = \min(\max(0, \widetilde t), 1)$, where
$$\widetilde{t}:=\frac{n}{(n-2)}\bigg(\frac{\|\text{diag}(C\mathbf{K}C)\|_{F}^{2}-\frac{1}{n}\|C\mathbf{K}C\|_{F}^{2}}{\|C\mathbf{K}C\|_{F}^{2}}\bigg).$$
Solving back for $\gamma$ gives the ridge penalty $\widehat{\gamma}=\widehat{t}/(1-\widehat{t})$. We call this approach Stabilization.

Generalized cross-validation (GCV) \cite{Craven1978, xiang1996generalized, golub1979generalized} is another common method for selection of ridge parameter, however we found that it performs poorly compared to proposed Stabilization method. Figure~\ref{fig:RidgeError} compares the selected ridge parameters as well as corresponding error rates for two methods. 
We generate 100 training and testing datasets following the model in Section \ref{sec:SimulatedData}. Each time we consider five possible kernel parameters $\sigma^2$ based on the distance quantiles as in Section~\ref{sec:KernParameter}. We then select ridge parameters by either GCV or proposed stabilization method, and choose the best sparsity parameter for each as in Section~\ref{sec:LASSOselect}. We find that GCV consistently selects smaller value for the ridge parameter than our approach leading to higher error rates. We conjecture that surprisingly poor performance of GCV is due to the presence of noise variables, although we do not have the formal justification. 

\begin{figure}[!t]
\centering
\includegraphics[scale=.75]{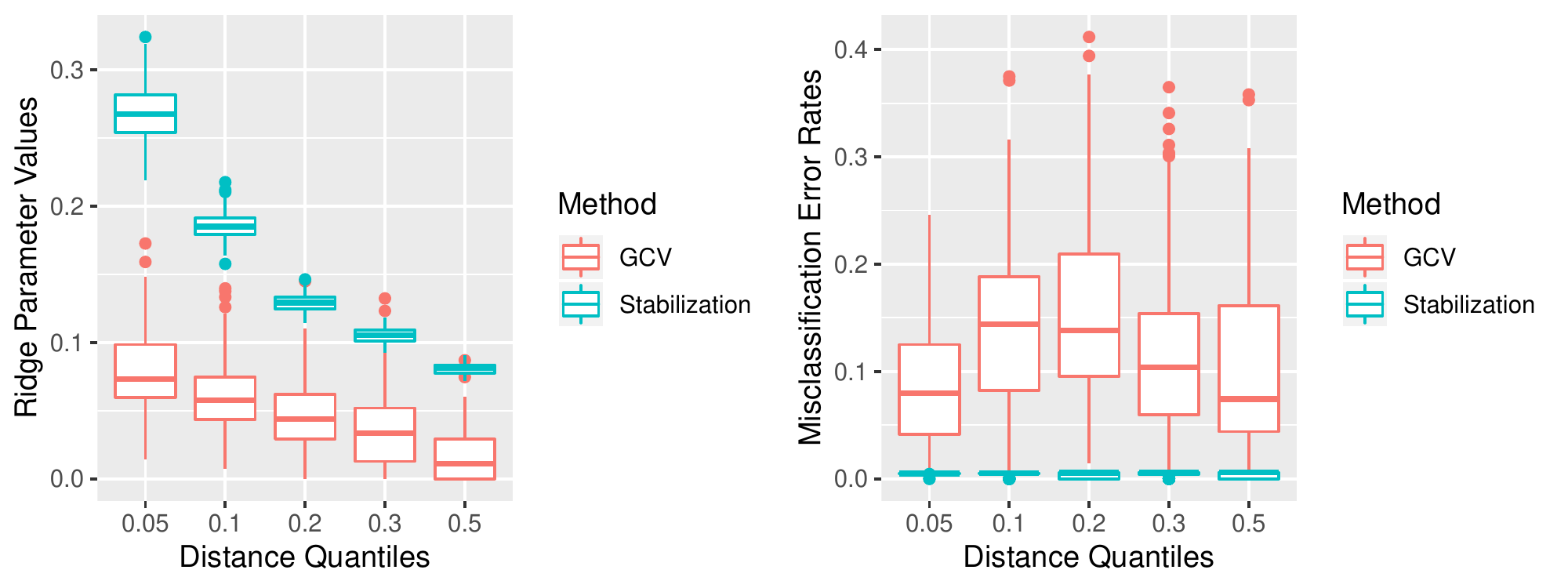}
\caption{Comparison between generalized cross-validation (GCV) and proposed Stabilization method for selection of ridge parameter $\gamma$ over 100 replications.
\textbf{Left:} Selected values of $\gamma$; \textbf{Right:} Misclassification error rates.}
\label{fig:RidgeError}
\end{figure}

\subsection{Sparsity parameter selection}\label{sec:LASSOselect}

We select $\lambda$ using 5-fold cross-validation (CV) to minimize the error rate over a grid of 20 equally-spaced values in $[10^{-10}\lambda_{\text{max}}, \lambda_\text{max}]$. We set $\lambda_{\max} = 2 \|\beta\|_\infty$, where $\beta$ is as in~\eqref{eq:QandB}, since the solution $\widehat w$ to~\eqref{eq:wlasso} is zero if $\lambda \geq \lambda_{\max}$ (see Lemma~1 in the Supplement).

%Let $\beta$ be as in~\eqref{eq:QandB} and let $\lambda_{\max} = 2 \|\beta\|_\infty$. The solution $\widehat w$ to~\eqref{eq:wlasso} is zero if $\lambda \geq \lambda_{\max}$. We select $\lambda$ using $5$-fold cross-validation over a grid of $20$ equally spaced values

%{\textbf{\textit{Usually we use the logarithmic grid in lasso context, I would not rerun any simulations results because of this, but will adjust the grid generation in R package to be logarithmic by default, see notes from sparsity class}}}
%{\color{red} Alex- Okay. Can work on that later. Let's leave this this comment up}

%on the interval $[10^{-10}\lambda_{\text{max}}, \lambda_\text{max}]$ to minimize the misclassification error rate.

%{\color{blue} rather than having the lemma, refer to some lasso literature for this, shorten the paper, note to myself}

\section{Empirical studies}\label{sec:EmpiricalStudies}

We compare the performance of the following methods: (i) sparse kernel optimal scoring (Sparse KOS); (ii) kernel optimal scoring (KOS); (iii) random forests; (iv) kernel support vector machines (kernel SVM); (v)~neural networks; (vi) K-nearest neighbors (KNN); and (vii) sparse linear discriminant analysis (sparse LDA).

We implement sparse KOS using the gaussian kernel with parameters selected as in Section \ref{sec:ParameterSelection}, KOS is implemented by setting $\lambda = 0$ and $w=\one$. We use the R package \textsf{randomForest} \cite{randomForest} to create a classifier with 50 decision trees. We use the R package \textsf{kernlab} \cite{karatzoglou2004kernlab} for kernel SVM using the gaussian kernel with parameter selected as in Section \ref{sec:KernParameter}. We use \textsf{keras} \cite{chollet2017kerasR} to implement a neural network with the ReLU activation function, 50 units, 100 epochs, and the default batch size.
We use \textsf{class} \cite{KNN} for KNN with $K=5$. We use the R package \textsf{MGSDA} \cite{Gaynanova:2016wk} for sparse LDA. 

\subsection{Simulated model 1}\label{sec:SimulatedData}
We generate data as in Figure~\ref{fig:TrainTestScatterplot} with $p=4$ features $(x_1, x_2, x_3, x_4)$. The first two features satisfy $\sqrt{x_{i1}^2 + x_{i2}^2}\ge 2/3$ if the $i$th sample is in class 1, and $\sqrt{x_{i1}^2 + x_{i2}^2}\le 2/3 - 1/10$ if the $i$th sample is in class 2. We generate $300$ samples with each feature from the uniform distribution on $[-1,1]$ and only leave samples that satisfy one of the class requirements ($n\approx 270$). The remaining two features are generated as independent gaussian noise variables, $x_{ij}\sim \Ncal (0, 2^{-1})$ for $j=3,4$ and all samples $i$. We use 2/3 of the samples for training, and 1/3 for testing, maintaining the class proportions. We repeat the data generation process and the split 100 times, the misclassification error rates over test datasets are presented in Figure~\ref{fig:ErrorScatterplot}.

%We construct the data set as follows:
%\begin{enumerate}
%\item Each data point has four coordinates $(x_i, y_i, \varepsilon_i^1, \varepsilon_i^2)$. The first two coordinates $(x_i, y_i)\sim \text{Unif}\,[-1,1]^{2}$ are useful for group classification, and the last two $\varepsilon_i^1, \varepsilon_i^2\sim \mathcal{N}(\mu=0, \sigma^2=\frac{1}{2})$ are independent gaussian noise variables. 
 
%\item All data points such that
%$\sqrt{x_{i}^{2}+y_{i}^{2}}\leq 2/3-1/10\approx .566$ %belong to category $1$. All data points
%such that $\sqrt{x_{i}^{2}+y_{i}^{2}}\geq 2/3$ belong to %category $2$. 

%\item All data points such that $2/3-1/10< \sqrt{x_i^2+y_i^2}<2/3$ are removed. Approximately 270 data points will remain.

%\item We take stratified samples of training and test data so that class proportions are maintained. Two-thirds of the data in both classes are randomly sampled to form the training data, and the rest of the data is used as testing data.

%\end{enumerate}

\begin{figure}
\centering
\begin{subfigure}[t]{.5\linewidth}
  \centering
  \includegraphics[width=7cm]{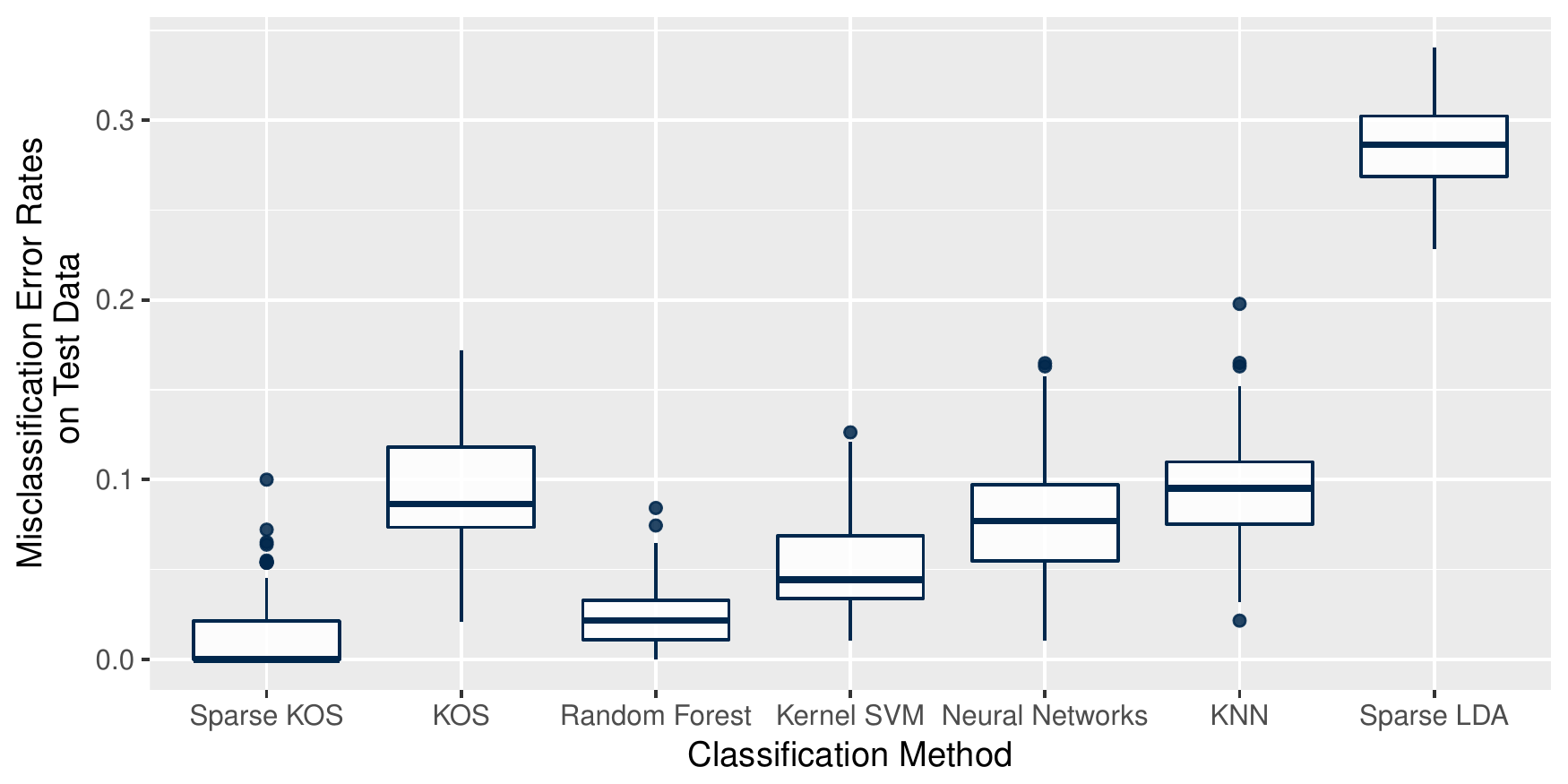}
\end{subfigure}%
\begin{subfigure}[t]{.5\linewidth}
 \includegraphics[width=7cm]{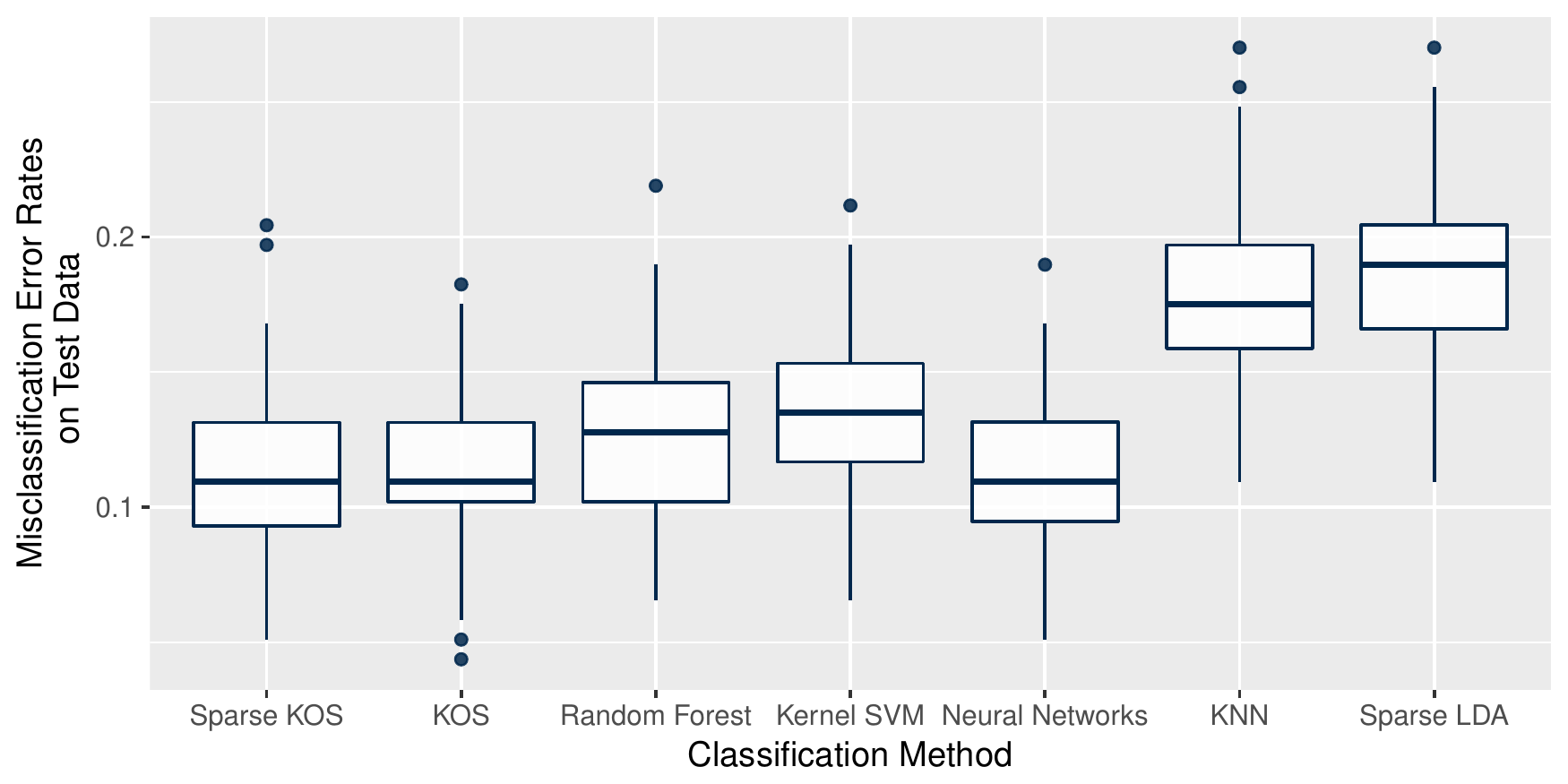}
\end{subfigure}
\caption{\textbf{Left: }Misclassification error rates based on 100 replications of simulated model 1.
\textbf{Right: }Misclassification error rates based on 100 replications of simulated model 2.}
\label{fig:ErrorScatterplot}
\end{figure}

%\begin{figure}
%\centering
%\includegraphics[width=8cm]{Model1Errors.pdf}
%\caption{Misclassification error rates based on 100 replications of simulated model 1.}
%\label{fig:ErrorBoxplot}
%\end{figure}

Sparse KOS performs the best out of all classifiers with random forest being second-best. Sparse LDA performs the worst, likely due to non-linear optimal classification boundary. Sparse KOS has excellent feature selection in this study- giving nonzero weight to the first two features in all $100$ splits while giving $\widehat w_j = 1$ for $j=1,2$ in 98 out of 100 replications and $\widehat w_j = 0$ for $j=3,4$ in $99$ out of 100 replications.

%\begin{figure}[t]
%\includegraphics[width=8cm]{Model1Size.pdf}
%\caption{Average of the absolute values of the weight values for each feature across the 100 independent simulations of model 1. Bars represent plus or minus twice the standard error.}
%\label{fig:WeightsSim1}
%\end{figure}

%The results show that sparse kernel optimal scoring out performs the six other non-parameteric classifiers. The median misclassification error rate for sparse KOS is $0.00\%$, and the upper quartile error rate is $1.11\%$. By comparison, the lower quartile error rate for random forest classification is $1.08\%$ and the median is $2.15\%$. Sparse linear discriminant analysis has a median error rate of $28.65\%$. Kernel SVM has a median error rate of $5.38\%$, while KOS has a median error rate of $8.60\%$. Neural Networks have a median error rate of $7.53\%$.

\subsection{Simulated model 2}
We generate data with $p=10$ features and $n=400$ samples such that $x_{i3}+\sin(x_{i4}+x_{i1})<(x_{i2})^2$ if sample $i$ belongs to class 1, and $x_{i3}+\sin(x_{i4}+x_{i1})\ge (x_{i2})^2$ if sample $i$ belongs to class 2. We use the uniform distribution on $[-1,1]$ for each $x_{ij}$, so that the last 6 features are uniform noise.
%\begin{enumerate}
%\item Let $x_{j}\in \text{Unif}[-1,1]$ for $j=1, \dots, p$. Generate $Y\in \mathbb{R}^{n}$ as $y_{i}=1$ if
%\[
%x_{3}^{i}+\sin(x_{4}^{i}+x_{1}^{i})<(x_2^{i})^2
%\]
%and $2$ otherwise.
%\end{enumerate}
As with the previous example, we use 2/3 of the samples for training, and 1/3 for testing, where the split is performed to maintain the class proportions. We repeat the data generation process and the split 100 times. The misclassification error rates over test datasets are presented in Figure~\ref{fig:ErrorScatterplot}.

%\begin{figure}[!t]
%\centering
%\includegraphics[width=8cm]{Model2Errors.pdf}
%\caption{Misclassification error rates based on 100 replications of simulated model 2.}
%\label{fig:SecondErrorScatterplot}
%\end{figure}

The lowest misclassification error rates are achieved by sparse KOS, KOS, and neural network classifiers. Sparse KOS behaves similarly to KOS because sparse KOS is unable to consistently select true features. Nevertheless, it gives higher weight values to true features as displayed in Figure \ref{fig:WeightsSim2}. As with the previous example, sparse LDA performs the worst.
%due to the classification boundary being non-linear.

\begin{figure}[!t]
\begin{center}
\includegraphics[width=8cm]{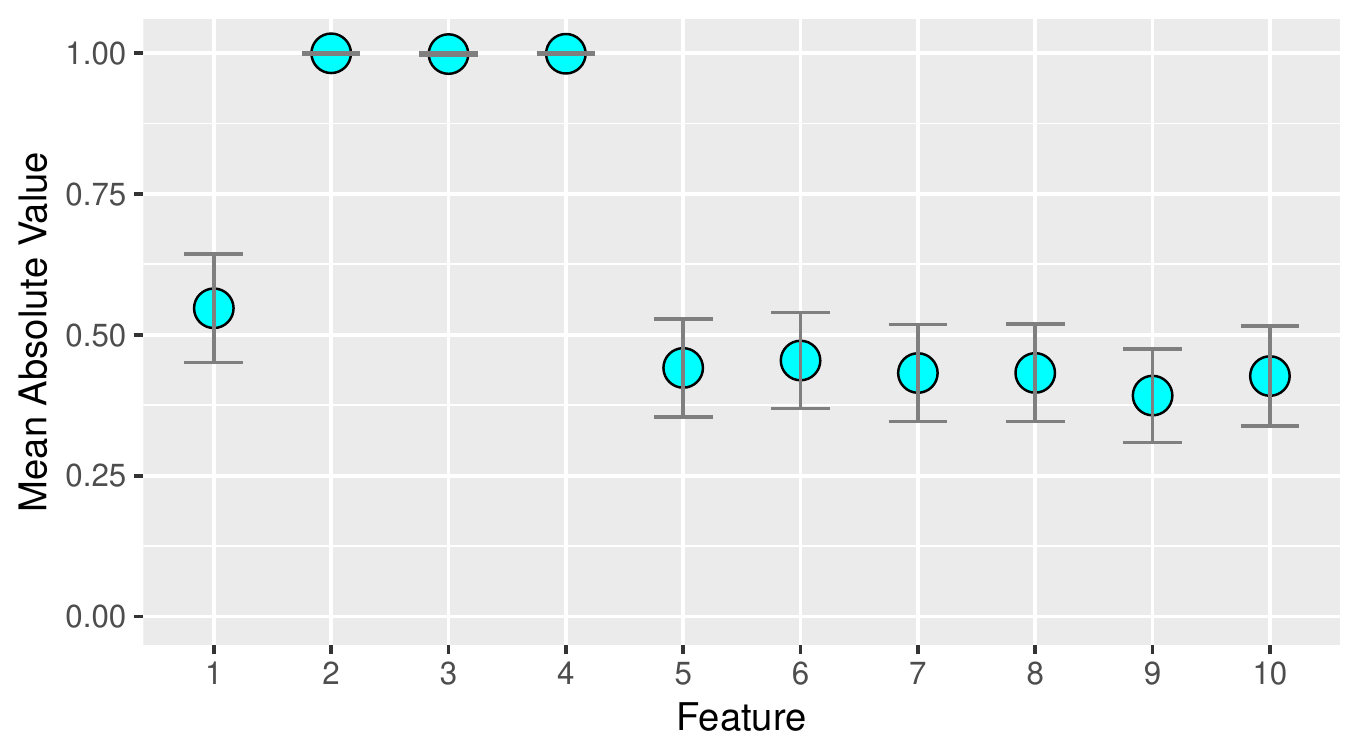}
\caption{The mean absolute values of weights $|w_j|$ for each feature across 100 replications of simulated model 2. The bars represent $\pm 2$ standard errors.}
\label{fig:WeightsSim2}
\end{center}
\end{figure}

\subsection{Benchmark datasets}\label{sec:RealSimulations}

We consider three datasets, summarized in Table~\ref{tab:DataSummary}, which are publicly available from the UCI Machine Learning Repository. We randomly split each dataset 100 times preserving the class proportions and use 2/3 for training and 1/3 for testing. We do not present the error rates for sparse LDA due to its poor performance on these datasets (it classifies every point to the largest of two groups), the misclassification error rates for all other methods are in Table~\ref{tab:Errors}. 

\begin{table}[!t]
\begin{center}
    \begin{tabular}{ | p{3.9cm} | p{1.55cm} | p{1.55cm}  |}
    \hline
    Dataset & Features size & Sample size \\ \hline
    Blood donation  \cite{yeh2009knowledge} & $p=4$ & $n=748$  \\ \hline
    Climate model failure \cite{lucas2013failure} & $p=18$ & $n=540$   \\ \hline
    Credit card default \cite{yeh2009comparisons} & $p=24$ & $n=3,000$ \\
    \hline
    \end{tabular}
     \caption{Description of benchmark datasets} \label{tab:DataSummary} 
\end{center}
\end{table}

\begin{table}[!t]
\begin{center}
    \begin{tabular}{|p{2.5cm} | p{1.7cm} |p{1.5cm} |p{1.7cm} | }
    \hline
    {}& {Blood $\quad$ Donation} & Climate Model & Credit$\quad$  Default\\ \hline
    Sparse KOS & \textbf{22.1} (0.18) & \textbf{4.9} (0.13) & \textbf{18.2} (0.06) \\ \hline
    KOS &\textbf{22.2} (0.20) & 5.4 (0.12) & 19.1 (0.08)\\ \hline
    Random Forest & 24.3 (0.18)& 8.2 (0.06) & 19.1 (0.08)\\ \hline
    Kernel SVM& 22.4 (0.12)&8.7 (0.00)& 20.0 (0.08)\\ \hline
    Neural Network& 23.9 (0.04) &5.4 (0.15)& 21.7 (0.04)\\ \hline
    KNN & 23.5 (0.20) & 7.6 (0.08)& 20.8 (0.08)\\ \hline
    \end{tabular}
     \caption{Mean misclassification errors (\%) over 100 random splits, standard errors are in brackets.} 
      \label{tab:Errors}
\end{center}
\end{table}

In the blood donation study \cite{yeh2009knowledge}, the goal is to determine if a person will donate blood given four features: Recency (months since last donation), Frequency (total number of donations), 
Monetary (total blood donated in cubic centimetres), and Time since first donation. Sparse KOS consistently gives large weights ($|w_j|>0.9$) to every feature but Frequency. The latter gets large weight in only 50\% of splits. %but low or zero weight in the remaining splits. 
Sparse KOS performs similarly to KOS, and we conjecture this is because all features are important for classification.

%\begin{figure}[!t]
%\centering
%\includegraphics[width=8cm]{BloodSize.pdf}
%\caption{ Average of the absolute values of the weight values based on 100 replications of the Blood Donation simulation. Error bars indicate plus or minus two standard errors of the mean.}
%\label{fig:BloodWeights}
%\end{figure}

%\begin{figure}[!t]
%\centering
%\includegraphics[width=8cm]{BloodErrors.pdf}
%\caption{ Misclassification error rates based on 100 replications for the blood donation data %set.}
%\label{fig:BloodError}
%\end{figure}

%The error rates for 100 iterations are shown in Figure \ref{fig:CMFError}. Figure \ref{fig:CMFSize} shows a boxplot of the model sizes over those $100$ iterations. 
%The median number of nonzero coefficients used in sparse KOS is $7$.

%\begin{figure}[!t]
%\centering
%\includegraphics[width=8cm]{CMFErrors.pdf}
%\caption{Misclassification error rates based on 100 replications of the climate model failure %simulation data.}
%\label{fig:CMFError}
%\end{figure}

In the climate model study \cite{lucas2013failure}, the goal is to predict if a climate simulation will crash based on 18 initial parameter values. Sparse KOS consistently selects 4 out of 18: features 1, 2 (variable viscosity parameters), feature 13 (tracer and momentum mixing coefficient), and  feature 14 (base background vertical diffusivity). Sparse KOS has the best classification performance, which is likely due to feature selection.

The credit card data \cite{yeh2009comparisons} has 30,000 data points, but we restrict to $n=3,000$ for computational simplicity. The goal is to predict the default of a customer on credit payments based on 24 features. Sparse KOS has the best classification performance, followed by KOS and random forests. Sparse KOS always selects feature 6 (the repayment status in September, 2005, the latest monthly payment recorded) and rarely selects other features. %. Features 6-11 are records of past payments from April 2005 to September 2005. Sparse KOS is This indicates that 
The most recent payment history is strongly indicative of credit default.

\section{Discussion}

We propose a kernel discriminant classifier with sparse feature selection, called sparse kernel optimal scoring, which is implemented in the R package \texttt{sparseKOS} \cite{sparseKOS}. An advantage of sparsity is that it can improve classification performance (see Section~\ref{sec:EmpiricalStudies}) and lead to more interpretable classification rules. The nonzero weights produced by sparse KOS can be used to judge the importance of features. While we have focused the discussion on the case of two classes, the method can be generalized to multiple classes using optimal scoring formulation in~\cite{gaynanova2018prediction}.

Sparse KOS requires the construction of a $n\times n$ kernel matrix $\mathbf{K}$ and is therefore computationally prohibitive for large $n$ cases. Future research could investigate the appropriate low-dimensional approximations of $\mathbf{K}$ within the kernel optimal scoring framework.

\vspace{0.05in}

\section*{Acknowledgements:} This work was supported in part by \uppercase{NSF-DMS} 1712943.

\bibliographystyle{plain}
\bibliography{AlexReferences,IrinaReferences}

\begin{thebibliography}{}

\bibitem[Allen, 2013]{allen_automatic_2013}
Allen, G.~I. (2013).
\newblock Automatic feature selection via weighted kernels and regularization.
\newblock {\em Journal of Computational and Graphical Statistics},
  22(2):284--299.

\bibitem[Bach, 2008]{bach2008consistency}
Bach, F.~R. (2008).
\newblock Consistency of the group lasso and multiple kernel learning.
\newblock {\em Journal of Machine Learning Research}, 9(Jun):1179--1225.

\bibitem[Bach et~al., 2004]{bach2004multiple}
Bach, F.~R., Lanckriet, G.~R., and Jordan, M.~I. (2004).
\newblock Multiple kernel learning, conic duality, and the smo algorithm.
\newblock In {\em Proceedings of the twenty-first international conference on
  Machine learning}, page~6. ACM.

\bibitem[Baudat and Anouar, 2000]{baudat2000generalized}
Baudat, G. and Anouar, F. (2000).
\newblock Generalized discriminant analysis using a kernel approach.
\newblock {\em Neural Computation}, 12(10):2385--2404.

\bibitem[Bousquet et~al., 2004]{bousquet2004introduction}
Bousquet, O., Boucheron, S., and Lugosi, G. (2004).
\newblock Introduction to statistical learning theory.
\newblock In {\em Advanced lectures on machine learning}, pages 169--207.
  Springer.

\bibitem[Boyd and Vandenberghe, 2004]{boyd2004convex}
Boyd, S. and Vandenberghe, L. (2004).
\newblock {\em Convex optimization}.
\newblock Cambridge university press.

\bibitem[Cai and Liu, 2011]{Cai:2011dm}
Cai, T. and Liu, W. (2011).
\newblock {A direct estimation approach to sparse linear discriminant
  analysis}.
\newblock {\em Journal of the American Statistical Association},
  106(496):1566--1577.

\bibitem[Caputo et~al., 2002]{caputo2002appearance}
Caputo, B., Sim, K., Furesjo, F., and Smola, A. (2002).
\newblock Appearance-based object recognition using svms: which kernel should i
  use?
\newblock In {\em Proceedings of NIPS Workshop on Statistical Methods for
  Computational Experiments in Visual Processing and Computer Vision,
  Whistler}, volume 2002.

\bibitem[Chen et~al., 2017]{chen2017double}
Chen, J., Zhang, C., Kosorok, M.~R., and Liu, Y. (2017).
\newblock Double sparsity kernel learning with automatic variable selection and
  data extraction.
\newblock {\em arXiv preprint arXiv:1706.01426}.

\bibitem[Chollet et~al., 2017]{chollet2017kerasR}
Chollet, F., Allaire, J., et~al. (2017).
\newblock R interface to keras.
\newblock \url{https://github.com/rstudio/keras}.

\bibitem[Clemmensen et~al., 2011]{Clemmensen:2011kr}
Clemmensen, L., Witten, D.~M., Hastie, T., and Ersb{\o}ll, B. (2011).
\newblock {Sparse Discriminant Analysis}.
\newblock {\em Technometrics}, 53(4):406--413.

\bibitem[Craven and Wahba, 1978]{Craven1978}
Craven, P. and Wahba, G. (1978).
\newblock Smoothing noisy data with spline functions.
\newblock {\em Numerische Mathematik}, 31(4):377--403.

\bibitem[Diethe et~al., 2009]{diethe2009matching}
Diethe, T., Hussain, Z., Hardoon, D., and Shawe-Taylor, J. (2009).
\newblock Matching pursuit kernel fisher discriminant analysis.
\newblock In {\em Artificial Intelligence and Statistics}, pages 121--128.

\bibitem[Friedman et~al., 2009]{hastie_elements_2009}
Friedman, J., Hastie, T., and Tibshirani, R. (2009).
\newblock {\em The Elements of statistical learning}.
\newblock Springer Series in Statistics New York, 2 edition.

\bibitem[Gaynanova, 2018]{gaynanova2018prediction}
Gaynanova, I. (2018).
\newblock Prediction and estimation consistency of sparse multi-class penalized
  optimal scoring.
\newblock {\em arXiv preprint arXiv:1809.04669}.

\bibitem[Gaynanova et~al., 2016]{Gaynanova:2016wk}
Gaynanova, I., Booth, J.~G., and Wells, M.~T. (2016).
\newblock {Simultaneous sparse estimation of canonical vectors in the $p>> N$
  setting}.
\newblock {\em Journal of the American Statistical Association}, 111:696--706.

\bibitem[Gaynanova and Wang, 2017]{gaynanova2017sparse}
Gaynanova, I. and Wang, T. (2017).
\newblock Sparse quadratic classification rules via linear dimension reduction.
\newblock {\em arXiv preprint arXiv:1711.04817}.

\bibitem[Golub et~al., 1979]{golub1979generalized}
Golub, G.~H., Heath, M., and Wahba, G. (1979).
\newblock Generalized cross-validation as a method for choosing a good ridge
  parameter.
\newblock {\em Technometrics}, 21(2):215--223.

\bibitem[Hastie et~al., 1995]{hastie_penalized_1995}
Hastie, T., Buja, A., and Tibshirani, R. (1995).
\newblock Penalized discriminant analysis.
\newblock {\em The Annals of Statistics}, pages 73--102.

\bibitem[Hastie et~al., 2015]{hastie2015statistical}
Hastie, T., Tibshirani, R., and Wainwright, M. (2015).
\newblock {\em Statistical learning with sparsity: the lasso and
  generalizations}.
\newblock CRC press.

\bibitem[Hastie et~al., 1994]{Hastie:1994cx}
Hastie, T., Tibshirani, R.~J., and Buja, A. (1994).
\newblock {Flexible discriminant analysis by optimal scoring}.
\newblock {\em Journal of the American Statistical Association},
  89(428):1255--1270.

\bibitem[Hein and Bousquet, 2004]{hein2004kernels}
Hein, M. and Bousquet, O. (2004).
\newblock Kernels, associated structures and generalizations.
\newblock {\em Max-Planck-Institut fuer biologische Kybernetik, Technical
  Report}.

\bibitem[Karatzoglou et~al., 2004]{karatzoglou2004kernlab}
Karatzoglou, A., Smola, A., Hornik, K., and Zeileis, A. (2004).
\newblock Kernlab-an s4 package for kernel methods in r.
\newblock {\em Journal of Statistical Software}, 11(9):1--20.

\bibitem[Kim et~al., 2006]{kim2006robust}
Kim, S.-J., Magnani, A., and Boyd, S. (2006).
\newblock Robust fisher discriminant analysis.
\newblock In {\em Advances in Neural Information Processing Systems}, pages
  659--666.

\bibitem[Kimeldorf and Wahba, 1970]{kimeldorf_correspondence_1970}
Kimeldorf, G.~S. and Wahba, G. (1970).
\newblock A correspondence between bayesian estimation on stochastic processes
  and smoothing by splines.
\newblock {\em The Annals of Mathematical Statistics}, 41(2):495--502.

\bibitem[Lancewicki, 2017]{lancewicki_regularization_2017}
Lancewicki, T. (2017).
\newblock Regularization of the kernel matrix via covariance matrix shrinkage
  estimation.
\newblock {\em arXiv preprint arXiv:1707.06156}.

\bibitem[Lanckriet et~al., 2002]{lanckriet2002robust}
Lanckriet, G.~R., Ghaoui, L.~E., Bhattacharyya, C., and Jordan, M.~I. (2002).
\newblock A robust minimax approach to classification.
\newblock {\em Journal of Machine Learning Research}, 3(Dec):555--582.

\bibitem[Lapanowski and Gaynanova, 2018]{sparseKOS}
Lapanowski, A.~F. and Gaynanova, I. (2018).
\newblock {\em sparseKOS: An R package for Sparse Kernel Optimal Scoring}.
\newblock Available at https://github.com/aflapan/sparseKOS.

\bibitem[Liaw and Wiener, 2002]{randomForest}
Liaw, A. and Wiener, M. (2002).
\newblock Classification and regression by randomforest.
\newblock {\em R News}, 2(3):18--22.

\bibitem[Lucas et~al., 2013]{lucas2013failure}
Lucas, D., Klein, R., Tannahill, J., Ivanova, D., Brandon, S., Domyancic, D.,
  and Zhang, Y. (2013).
\newblock Failure analysis of parameter-induced simulation crashes in climate
  models.
\newblock {\em Geoscientific Model Development}, 6(4):1157--1171.

\bibitem[Mika et~al., 1999]{mika_fisher_1999}
Mika, S., Ratsch, G., Weston, J., Scholkopf, B., and Mullers, K.-R. (1999).
\newblock Fisher discriminant analysis with kernels.
\newblock In {\em Neural networks for signal processing IX, 1999. Proceedings
  of the 1999 IEEE Signal Processing Society Workshop.}, pages 41--48. IEEE.

\bibitem[Nosedal-Sanchez et~al., 2012]{RKHS_tutorial2012}
Nosedal-Sanchez, A., Storlie, C.~B., Lee, T.~C., and Christensen, R. (2012).
\newblock Reproducing kernel hilbert spaces for penalized regression: A
  tutorial.
\newblock {\em The American Statistician}, 66(1):50--60.

\bibitem[Roth and Steinhage, 2000]{roth2000nonlinear}
Roth, V. and Steinhage, V. (2000).
\newblock Nonlinear discriminant analysis using kernel functions.
\newblock In {\em Advances in Neural Information Processing Systems}, pages
  568--574.

\bibitem[Sch{\"o}lkopf and Smola, 2002]{bernhard_scholkopf_learning_2002}
Sch{\"o}lkopf, B. and Smola, A.~J. (2002).
\newblock {\em Learning with kernels: support vector machines, regularization,
  optimization, and beyond}.
\newblock MIT Press.

\bibitem[Sonnenburg et~al., 2006]{sonnenburg2006large}
Sonnenburg, S., R{\"a}tsch, G., Sch{\"a}fer, C., and Sch{\"o}lkopf, B. (2006).
\newblock Large scale multiple kernel learning.
\newblock {\em Journal of Machine Learning Research}, 7(Jul):1531--1565.

\bibitem[Steinwart and Scovel, 2007]{steinwart_fast_2007}
Steinwart, I. and Scovel, C. (2007).
\newblock Fast rates for support vector machines using gaussian kernels.
\newblock {\em The Annals of Statistics}, pages 575--607.

\bibitem[Sun et~al., 2015]{sun2015learning}
Sun, S., Kolar, M., and Xu, J. (2015).
\newblock Learning structured densities via infinite dimensional exponential
  families.
\newblock In {\em Advances in Neural Information Processing Systems}, pages
  2287--2295.

\bibitem[Venables and Ripley, 2002]{KNN}
Venables, W.~N. and Ripley, B.~D. (2002).
\newblock {\em Modern Applied Statistics with S}.
\newblock Springer, New York, fourth edition.

\bibitem[Xiang and Wahba, 1996]{xiang1996generalized}
Xiang, D. and Wahba, G. (1996).
\newblock A generalized approximate cross validation for smoothing splines with
  non-gaussian data.
\newblock {\em Statistica Sinica}, pages 675--692.

\bibitem[Yeh and Lien, 2009]{yeh2009comparisons}
Yeh, I.-C. and Lien, C.-h. (2009).
\newblock The comparisons of data mining techniques for the predictive accuracy
  of probability of default of credit card clients.
\newblock {\em Expert Systems with Applications}, 36(2):2473--2480.

\bibitem[Yeh et~al., 2009]{yeh2009knowledge}
Yeh, I.-C., Yang, K.-J., and Ting, T.-M. (2009).
\newblock Knowledge discovery on rfm model using bernoulli sequence.
\newblock {\em Expert Systems with Applications}, 36(3):5866--5871.

\bibitem[Zhang et~al., 2016]{zhang2016quantile}
Zhang, C., Liu, Y., and Wu, Y. (2016).
\newblock On quantile regression in reproducing kernel hilbert spaces with data
  sparsity constraint.
\newblock {\em Journal of Machine Learning Research}, 17(40):1--45.

\end{thebibliography}
  
\onecolumn 

\appendix

\section{Derivation of projection formula~\eqref{eq:Projection}}

\begin{proof}
 Since $\widehat f = \sum_{i=1}^{n} \widehat \alpha_{i}[\Phi(x_i)-\overline{\Phi}]$,
\begin{align*}
\left<\Phi(x)-\overline{\Phi}, \widehat f \,\right>_\Hcal&=\left<\Phi(x)-\overline{\Phi}\,,\, \sum_{i=1}^{n}\widehat \alpha_{i}[\Phi(x_i)-\overline{\Phi}]\right>_\Hcal\\
&=\sum_{i=1}^n\widehat \alpha_i \left<\Phi(x)-\overline{\Phi}, \Phi(x_i)-\overline{\Phi}\right>_\Hcal\\
&=\sum_{i=1}^n\widehat \alpha_i\left<\Phi(x), \Phi(x_i)\right>_\Hcal- \sum_{i=1}^n\widehat \alpha_i \left<\Phi(x), \overline{\Phi}\right>_\Hcal- \sum_{i=1}^n\widehat \alpha_i \left<\overline{\Phi}, \Phi(x_i)\right>_\Hcal+ \sum_{i=1}^n\widehat \alpha_i \left<\overline{\Phi}, \overline{\Phi}\right>_\Hcal\\
&=\sum_{i=1}^n\widehat \alpha_ik(x, x_i) - (\one^{\top}\widehat \alpha)\frac1{n}\sum_{i=1}^nk(x, x_i) - \frac1{n}\sum_{i=1}^n\sum_{j=1}^n\widehat \alpha_i k(x_j, x_i) + (\one^{\top}\widehat \alpha) \frac1{n^2}\sum_{i=1}^n\sum_{j=1}^n k(x_i, x_j).
\end{align*}
Let $K(X,x):=\begin{pmatrix}
k(x_1, x)&\cdots&k(x_n, x)
\end{pmatrix}^{\top}$. Then from the above display
\begin{align*}
    \left<\Phi(x)-\overline{\Phi}, \widehat f \,\right>_\Hcal &=K(X,x)^{\top}\widehat \alpha - n^{-1}K(X,x)^{\top}\one\one^{\top}\widehat \alpha - n^{-1}\one^{\top}K\widehat \alpha + \frac1{n^2}\one^{\top}K\one(\one^{\top}\widehat \alpha)\\
    &=K(X,x)^{\top}C\widehat \alpha - \frac1{n}\one^{\top}KC\widehat \alpha\\
    &= (K(X,x)^{\top} - \frac1{n}\one^{\top}K)C\widehat \alpha,
\end{align*}
where $C = I - n^{-1}\one\one^{\top}$ is the centering matrix.
%Let us now consider 
%\begin{align*}
%\left<\overline{\Phi}\,,\,\sum_{i=1}^{n}\alpha_{i}[\Phi(x_i)-\overline{\Phi}]\right>_\Hcal&=
%\left<\overline{\Phi}\,,\, \sum_{i=1}^{n}\alpha_{i} \Phi(x_i)\right>_{\Hcal}-\left<\overline{\Phi}\,,\, \sum_{i=1}^{n}\alpha_{i} \overline{\Phi}\right>_{\Hcal}\\
%&=\frac{1}{n}\sum_{j=1}^{n}\sum_{i=1}^{n} \alpha_{i} k(x_j, x_i)-\left<\overline{\Phi}, \overline{\Phi}\right>_\Hcal \mathbf{1}^{\top}\alpha\\
%&=\sum_{i=1}^{n}\bigg[\frac{1}{n}\sum_{j=1}^{n} k(x_j, x_i)\bigg]\alpha_{i}-
%\frac{1}{n^2}\sum_{i,j=1}^{n}k(x_i, x_j) \mathbf{1}^{\top}\alpha\\
%&=\frac{1}{n}\mathbf{1}^{\top} \mathbf{K}\alpha-(\frac{1}{n}\mathbf{1}^{\top}\mathbf{K}\mathbf{1})\frac{1}{n}\mathbf{1}^{\top}\alpha=\frac{1}{n}\mathbf{1}^{\top}\mathbf{K}\bigg(I-\frac{1}{n}\mathbf{1}\mathbf{1}^{\top}\bigg)\alpha=\frac{1}{n}\mathbf{1}^{\top}\mathbf{K}C\alpha
%\end{align*}
%The projection formula follows
%$$
%\left<\Phi(x)-\overline{\Phi}, \widehat %f\right>_\Hcal=
%K(X,x)^{\top}C\alpha
%-\frac{1}{n}\mathbf{1}^{\top}\mathbf{K}C\alpha=
%\bigg(K(X,x)^{\top}-\frac{1}{n}\mathbf{1}^{\top}\mathbf{K}\bigg)C\alpha.
%$$
\end{proof}

\section{Technical Proofs}\label{sec:ProofOfTheorems}

In this section we prove the results stated within the main text. We use $C$, $C_1$, $C_2$, $\dots$ to denote absolute positive constants that do not depend on the sample size $n$ but which may depend on $\|\theta^*\|_\infty, \kappa,$ or $\tau$. Their values may change from line to line. The dependence between the main Theorems and supplementary results is depicted below.

%\begin{figure}
\begin{center}
\begin{tikzpicture}[node distance=3cm, auto]
 
% place nodes
\node [block] (init) {Theorem \ref{thm:EstErrorProbBound}} ;
\node [block, left of=init](block-100) {Theorem \ref{thm:EmpRiskProbBound}} ;
\node [block, below of=block-100] (block-8) {Theorem \ref{thm:DifferenceRiskSupConverge}} ;
\node [block, left of=block-8] (block-1) {Theorem \ref{thm:SymBound}} ;
\node [block, right of=block-8] (block-2) {Theorem \ref{thm:ModRiskDiffConverge}} ;
\node [block, right of=block-2] (block-3) {Theorem \ref{thm:ModConsistency}} ;
\node [block, below of=block-1] (block-4) {Lemma \ref{lem:MaxBound}} ;
\node [block, left of=block-4] (block-5) {Lemma~\ref{lem:EpsilonNetInSymmetrization}} ;
\node [block, below of=block-2] (block-14) {Lemma \ref{lem:ytheta}} ;
\node [block, below of=block-8](block-9) {Lemma \ref{lem:OptScoreSupNormBound}} ;
\node [block, below of=block-9] (block-10) {Lemma \ref{lem:CoveringNum}} ;
\node [block, below of=block-4] (block-13) {Lemma \ref{lem:EpslnNetRestriction}} ;
\node [block, left of=block-13] (block-7) {Lemma  \ref{lem:ClosenessOfEmpRisk}} ;

\path [line] (block-1) -- (init) ;
\path [line] (block-2) -- (init) ;
\path [line] (block-3) -- (init) ;
\path [line] (block-8) -- (block-100) ;
\path [line] (block-7) -- (block-5) ;
\path [line] (block-10) -- (block-4) ;
\path [line] (block-1) -- (block-100) ;
\path [line] (block-4) -- (block-1) ;
\path [line] (block-14) -- (block-8);
\path [line] (block-9) -- (block-8);
\path [line] (block-5) -- (block-1);
\path [line] (block-8) -- (block-2);
\path [line] (block-14) -- (block-2);
\path [line] (block-13) -- (block-5);
\end{tikzpicture}
%\caption{Proof charts for Theorems \ref{thm:EstErrorProbBound} and \ref{thm:EmpRiskProbBound}. }
\end{center}
%\end{figure}

\subsection{Proofs of Theorems 1 and 2}

%\begin{theorem}\label{thm:EstErrorProbBound}
%Under assumptions \ref{assump:KernBound}, \ref{assump:ScoreBound}, and \ref{asusmp:Separable}, 
%for large enough $n$ there exist $C>0$ and $\widetilde C>0$ such that the following inequality holds with probability at least $1-\eta$
%$$
%R(\widehat f, \widehat \beta) \leq R(f^*, \beta^*) + \bigg\{ \frac{C \log(\widetilde C/\eta)}{n}\bigg\}^{1/2}.
%$$
%\end{theorem}

\begin{proof}[Proof of Theorem \ref{thm:EstErrorProbBound}]
Consider
\begin{align*}
R(\widehat{f},\widehat{\beta})- R(f^\ast,\beta^\ast)=\underbrace{R(\widehat{f}, \widehat{\beta})-\widetilde{R}_{\text{emp}}(\widehat{f}, \widehat{\beta})}_{I_1}+\underbrace{\widetilde{R}_{\text{emp}}(\widehat{f}, \widehat{\beta})-\widetilde{R}_{\text{emp}}(\widetilde{f}, \widetilde{\beta})}_{I_2}
+\underbrace{\widetilde{R}_{\text{emp}}(\widetilde{f}, \widetilde{\beta})- R(f^\ast,\beta^\ast)}_{I_3}.
\end{align*}
By the union bound and de Morgan's law,
\begin{align*}
\mathbb{P}\Big(R(\widehat f, \widehat \beta)-R(f^*,\beta^*)>\varepsilon\Big)\,\leq \,\mathbb{P}\Big( I_1 >\frac{\varepsilon}{3}\Big)+\mathbb{P}\Big( I_2 >\frac{\varepsilon}{3}\Big)+
\mathbb{P}\Big( I_3 >\frac{\varepsilon}{3}\Big).
\end{align*}

Applying Theorems~\ref{thm:SymBound},~\ref{thm:ModRiskDiffConverge} and~\ref{thm:ModConsistency} to $I_1$, $I_2$ and $I_3$ correspondingly, there exist constants $C,C_i >0$ such that
\begin{align*}
\mathbb{P}\Big(R(\widehat{f},\widehat{\beta})&- R(f^\ast,\beta^\ast)>\varep\Big)\\
&\leq 2\mathcal{N}_{\varep}\exp\Big(-  \frac{n\varepsilon^2}{128(\|\theta^*\|_\infty+\kappa \tau)^{4}}\Big)+C_2\exp\Big( -\frac{C_3n \varepsilon^2}{1+(\kappa\tau)^2}\Big)+2\exp\Big(-\frac{n \varepsilon^2}{16(\|\theta^*\|_\infty+\kappa \tau)^4}\Big)\\
&\leq C_4\, \mathcal{N_\varepsilon}\exp\Big(-\frac{C_5 n \varepsilon^2}{(\|\theta^*\|_\infty+\kappa \tau)^4}\Big), 
\end{align*}
where $\mathcal{N}_{\varep}=\{1+2(\|\theta^*\|_\infty+\kappa\tau)/\varep\} \exp( C\tau^2 \varepsilon^{-2})$. This concludes the proof of Theorem~\ref{thm:EstErrorProbBound}.
%sing Theorem~\ref{thm:SymBound}, $I_1 \leq C_1\frac{\log (C_2\eta^{-1})}{n}$ with probability at least $1-\eta/3$. Using Th

%Each of (I)-(III) will be given a probabilistic bound which holds with probability $1-\eta/3$. Theorem \ref{thm:SymBound} gives the bound for (I), Theorem \ref{thm:ModRiskDiffConverge} bounds (II), and Theorem \ref{thm:ModConsistency} bounds (III). Thus, the estimation error is bounded above by
%\[
%\bigg\{\frac{C_1 \log(\widetilde C_1/\eta)}{n}\bigg\}+
%\bigg\{\frac{C_2 \log(\widetilde C_2/\eta)}{n}\bigg\}+
%\bigg\{\frac{C_3 \log(\widetilde C_3/\eta)}{n}\bigg\}
%\]
%with probabiliy at least $1-\eta$. Let $C=3\max\{C_1, C_2, C_3\}$, and let
%$\widetilde{C}=\max\{\widetilde C_1, \widetilde C_2, \widetilde C_3\}.$ Then the estimation error is bounded above by
%\[
%R(\widehat{f},\widehat{\beta})- R(f^\ast,\beta^\ast)\leq 
%\bigg\{\frac{C \log(\widetilde C/\eta)}{n}\bigg\}.
%\]
%This proves the theorem. \qedhere
\end{proof}

%\begin{theorem}\label{thm:EmpRiskProbBound}
%Under assumptions \ref{assump:KernBound}, \ref{assump:ScoreBound}, and \ref{asusmp:Separable}, for large enough $n$ there exist $C>0$ and $\widetilde{C}>0$ such that the following inequality holds with probability at least $1-\eta$
%$$
%R(\widehat f, \widehat \beta) \leq R_{\text{emp}}(\widehat f) + \bigg\{\frac{C \log(\widetilde C/\eta)}{n} \bigg\}^{1/2}.
%$$
%\end{theorem}

\begin{proof}[Proof of Theorem \ref{thm:EmpRiskProbBound}]
Consider
\begin{align*}
    R(\widehat f, \widehat \beta)-R_{\text{emp}}(\widehat f)&=
      \underbrace{R(\widehat f, \widehat \beta)-\widetilde R_{\text{emp}}(\widehat f, \widehat \beta)}_{I_1}+\underbrace{\widetilde R_{\text{emp}}(\widehat f, \widehat \beta)-R_{\text{emp}}(\widehat f)}_{I_2}.
\end{align*}
By the union bound and de Morgan's law,
$$
\mathbb{P}\Big(  R(\widehat f, \widehat \beta)-R_{\text{emp}}(\widehat f)>\varep\Big)\leq \mathbb{P}\Big(I_{1}>\frac{\varep}{2} \Big)+\mathbb{P}\Big( I_2 >\frac{\varep}{2}\Big).
$$
%By Theorem~\ref{thm:SymBound}, there exist constants $C>0$ and $c>0$ such that $\mathbb{P}(I_{1}>\varep/2 )\leq 2 \mathcal{N}_{\varep}\exp(-c n \varepsilon^2)$, where $\mathcal{N}_{\varep}=\{1+2(\|\theta^*\|_\infty+\kappa\tau)/\varep\} \exp( C\tau^2 \varepsilon^{-2})$. By Theorem~\ref{thm:DifferenceRiskSupConverge}, there exists constants $\widetilde{C}>0$ and $\widetilde{c}>0$ such that $\mathbb{P}(I_2>\varep/2) \leq \widetilde{C}\exp(- \widetilde{c} n \varepsilon^2)$ for all $\varep>0.$
Applying Theorem~\ref{thm:SymBound} for $I_1$ and Theorem~\ref{thm:DifferenceRiskSupConverge} for $I_2$, the exist constants $C_i > 0$ such that
\begin{align*}
\mathbb{P}\Big(  R(\widehat f, \widehat \beta)-R_{\text{emp}}(\widehat f)>\varep\Big)&\leq 2\mathcal{N}_{\varep}\exp\Big(-  \frac{n\varepsilon^2}{128(\|\theta^*\|_\infty+\kappa \tau)^{4}}\Big)+ C_3 \exp\Big( -\frac{C_4n \varepsilon^2}{1+(\kappa\tau)^2}\Big)\\
&\leq C_5 \mathcal{N}_{\varep}\exp\Big(-\frac{C_6 n \varepsilon^2}{(\|\theta^*\|_\infty+\kappa \tau)^{4}}\Big),
\end{align*}
where $\mathcal{N}_{\varep}=\{1+2(\|\theta^*\|_\infty+\kappa\tau)/\varep\} \exp(C_1\tau^2 \varepsilon^{-2})$.
This concludes the proof of Theorem~\ref{thm:EmpRiskProbBound}. 
\end{proof}

\subsection{Supplementary Theorems}

\begin{theorem}\label{thm:SymBound}
Under Assumptions \ref{assump:ScoreBound}-\ref{assump:Separable}, there exists a constant $C_2 > 0$ such that for all $\varep>0$, 
$$
\mathbb{P}\Big(\sup_{f\in \mathcal{H}_\tau\,,\, \beta\in I_\tau}\{R(f,\beta)-\widetilde R_{\text{emp}}(f, \beta)\}>\varep\Big)
\leq  2\mathcal{N}_{\varep}\exp\Big(-  \frac{n\varepsilon^2}{128(\|\theta^*\|_\infty+\kappa \tau)^{4}}\Big),
$$
where $\mathcal{N}_{\varep}=\{1+2(\|\theta^*\|_\infty+\kappa\tau)/\varep\} 
\exp( C_2\tau^2 \varepsilon^{-2})$.
\end{theorem}

%The term $(II)$ is bounded by
\begin{theorem}\label{thm:ModRiskDiffConverge}
 Let $\widehat{\beta}=-\left<\,\overline{\Phi}, \widehat{f}\,\right>_\Hcal$. Under Assumptions \ref{assump:ScoreBound} and \ref{assump:KernBound}, there exist constants $C_1, C_2>0$ such that for all $\varep>0,$
 \[
\mathbb{P}\Big(\Big| \widetilde{R}_{\text{emp}}(\widehat{f}, \widehat{\beta})-\widetilde{R}_{\text{emp}}(\widetilde{f}, \widetilde{\beta})\Big|>\varep\Big) \leq C_1\exp\Big( -\frac{C_2n \varepsilon^2}{1+(\kappa\tau)^2}\Big).
 \]
\end{theorem}

%The term (III) is bounded by
\begin{theorem}\label{thm:ModConsistency}
Under Assumptions \ref{assump:ScoreBound} and \ref{assump:KernBound}, for all $\varep>0$
$$
\mathbb{P}\Big(\widetilde R_{\text{emp}}(\widetilde f, \widetilde \beta)-R(f^\ast, \beta^\ast) >\varep \Big)
\leq 2\exp\Big( -\frac{n\varepsilon^2}{16(\|\theta^*\|_\infty+\kappa\tau)^{4}}\Big).
$$
\end{theorem}

\begin{theorem}\label{thm:DifferenceRiskSupConverge}
Let Assumptions \ref{assump:ScoreBound} and \ref{assump:KernBound} be true, and let $\beta(f):=n^{-1}\sum_{i=1}^ny_i^{\top}\theta^*-\left<\overline{\Phi},f\right>_\Hcal= \overline{Y\theta^*} - \left<\overline{\Phi},f\right>_\Hcal$ be the minimizing $\beta\in I_\tau$ for fixed $f\in \mathcal{H}_{\tau}$ in the modified empirical risk. There exists constants $C_1,C_2>0$ such that for all $\varep>0$
\[
\mathbb{P}\Big(\sup_{f\in \mathcal{H}_\tau} |R_{\text{emp}}(f)-\widetilde{R}_{\text{emp}}(f, \beta(f))|>\varepsilon \Big)\leq C_1 \exp\Big( -\frac{C_2n \varepsilon^2}{1+(\kappa\tau)^2}\Big).
\]
\end{theorem}

\begin{definition} 
The \emph{empirical measure} $T_x$ with respect to  $\{x_i\}_{i=1}^{n}$ is defined as $T_{x}:=n^{-1}\sum_{i=1}^{n} \delta(x_i),$ where $\delta(x_i)$ is the point mass at $x_i$. The space $L^{2}(T_{x})$ is the set $\mathcal{H}_{\tau}$ equipped with the semi-norm
\[
\|f \|_{L^2(T_{x})}:=\sqrt{\frac{1}{n}\sum_{i=1}^{n} |f(x_i)|^{2}}=\sqrt{\frac{1}{n}\sum_{i=1}^{n}|\left<\Phi(x_i),f\right>_\Hcal|^{2}}.
\]
\end{definition}

\begin{definition}
Let $(X,d)$ be a pseudometric space. An $\varepsilon$-net is any subset $\widetilde{X}\subset X$ such that for any $x\in X$, there exists a $\widetilde{x}\in \widetilde{X}$ satisfying $d(x,\widetilde{x})<\varepsilon$. The $\varepsilon$-\emph{covering number} of $(X,d)$ is the minimum size of an $\varepsilon$-net for $X$.
\end{definition}

\begin{remark}
Distances in $\mathcal{H}_\tau$ are given by the semi-norm generated by $L^2(T_x)$. Distances in $I_\tau$ are given by the Euclidean distance $d(\beta_1, \beta_2)=|\beta_1-\beta_2|.$
\end{remark}

\subsection{Proofs of Supplementary Theorems}

\begin{proof}[Proof of Theorem \ref{thm:SymBound}]
Let $\{(x_j, y_j)\}_{j=n+1}^{2n}$ be independent from $\{(x_i, y_i)\}_{i=1}^{n}$ and identically distributed set of $n$ pairs, and let $T_x$ be the empirical measure on $\{(x_i, y_i)\}_{i=1}^{2n}$. Let $\widetilde{R}_{\text{emp}}(f,\beta)$ be the modified empirical risk on $\{(x_i, y_i)\}_{i=1}^{n}$, and $\widetilde{R}'_{\text{emp}}(f,\beta)$ on $\{(x_j, y_j)\}_{i=n+1}^{2n}$. By symmetrization lemma (see, for example, Lemma 2 in \cite{bousquet2004introduction}), for $n\varepsilon^2\geq 2$
\begin{align*}
\mathbb{P}\bigg( \sup_{f\in \mathcal{H}_\tau\,,\, \beta\in I_\tau}  \{R(f,\beta)-\widetilde{R}_\text{emp}(f,\beta)\}>\varepsilon\bigg)\leq 2\mathbb{P}\bigg( \sup_{f\in \mathcal{H}_\tau,\, \beta\in I_\tau}  \{\widetilde{R}_{\text{emp}}'(f,\beta)-\widetilde{R}_\text{emp}(f,\beta)\}>\frac{\varepsilon}{2}\bigg).
\end{align*}
Let $c = 64(\|\theta^*\|_{\infty} + \kappa \tau)$, and let $\{f_1, \dots, f_M\}$ be the smallest $L^2(T_x)$ $\varepsilon/\sqrt{2}c$-net of $\mathcal{H}_\tau$ and $\{\beta_1, \dots, \beta_K\}$ an $\varepsilon/c$-net of $I_\tau$. Applying Lemma~\ref{lem:EpsilonNetInSymmetrization} to the above display
\begin{align*}
\mathbb{P}\bigg( \sup_{f\in \mathcal{H}_\tau\,,\, \beta\in I_\tau}  \{R(f,\beta)-\widetilde{R}_\text{emp}(f,\beta)\}>\varepsilon\bigg)&\leq 2\mathbb{P}\bigg(\maximize_{\substack{f\in \{f_1, \dots, f_M\}\\
\beta\in \{\beta_1, \dots, \beta_K\}}}\{\widetilde{R}_{\text{emp}}'(f,\beta)-\widetilde{R}_\text{emp}(f,\beta)\}>\frac{\varepsilon}{4}\bigg).
\end{align*}
Applying Lemma~\ref{lem:MaxBound} to the right-hand expression gives the final inequality
\begin{align*}
&\mathbb{P}\bigg( \sup_{f\in \mathcal{H}_\tau\,,\, \beta\in I_\tau}  \{R(f,\beta)-\widetilde{R}_\text{emp}(f,\beta)\}>\varepsilon\bigg)\\
&\qquad\qquad\leq 2\{1+2(\|\theta^*\|_\infty+\kappa\tau)/\varep\} 
\exp\Big( \frac{C_1\tau^2}{\varepsilon^2}\Big) \exp\Big(-  \frac{n\varepsilon^2}{128(\|\theta^*\|_\infty+\kappa \tau)^{4}}\Big).
\end{align*}
This completes the proof of Theorem~\ref{thm:SymBound}.
\end{proof}

\begin{proof}[Proof of Theorem \ref{thm:ModRiskDiffConverge}]
Let $\beta(f) = \overline{Y\theta^*}-\left<\overline{\Phi},f\right>_\Hcal$. By definition of $\widetilde f$, $\widetilde \beta = \beta(\widetilde f)$, $\widetilde{R}_{\text{emp}}(\widehat{f}, \widehat{\beta}) \geq \widetilde{R}_{\text{emp}}(\widetilde {f}, \widetilde{\beta})$. On the other hand, since $R_{\text{emp}}(\widehat f)\leq R_{\text{emp}}(\widetilde f)$,
\begin{align*}
 \widetilde{R}_{\text{emp}}(\widehat{f}, \widehat{\beta})-\widetilde{R}_{\text{emp}}(\widetilde{f}, \widetilde{\beta})&= \widetilde{R}_{\text{emp}}(\widehat{f}, \widehat{\beta})-R_{\text{emp}}(\widehat{f})+R_{\text{emp}}(\widehat{f})-R_{\text{emp}}(\widetilde{f})+
 R_{\text{emp}}(\widetilde{f})-\widetilde{R}_{\text{emp}}(\widetilde{f}, \widetilde{\beta})\\
 &\leq \widetilde{R}_{\text{emp}}(\widehat{f}, \widehat{\beta})-R_{\text{emp}}(\widehat{f}) + R_{\text{emp}}(\widetilde{f})-\widetilde{R}_{\text{emp}}(\widetilde{f}, \widetilde{\beta})\\
&\leq \widetilde R_{\text{emp}}(\widehat f, \widehat \beta) - \widetilde R_{\text{emp}}(\widehat f, \beta(\widehat f)) + \widetilde R_{\text{emp}}(\widehat f, \beta(\widehat f)) -R_{\text{emp}}(\widehat{f})+
 R_{\text{emp}}(\widetilde{f})-\widetilde{R}_{\text{emp}}(\widetilde{f}, \widetilde{\beta})\\
&\leq 
\underbrace{\Big|\widetilde{R}_{\text{emp}}(\widehat{f}, \widehat{\beta})-\widetilde{R}_{\text{emp}}(\widehat{f}, \beta(\widehat f))\Big|}_{I_1}
+2\underbrace{\sup_{f\in \mathcal{H}_\tau} \Big| R_{\text{emp}}(f)-\widetilde{R}_{\text{emp}}(f, \beta(f))\Big|}_{I_2}.
\end{align*}
The union bound and de Morgan's law proves
$$
\mathbb{P}\Big(\widetilde{R}_{\text{emp}}(\widehat{f}, \widehat{\beta})-\widetilde{R}_{\text{emp}}(\widetilde{f}, \widetilde{\beta})>\varep\Big)\leq \mathbb{P}\Big( I_1>\frac{\varep}{2}\Big)+\mathbb{P}\Big( I_2>\frac{\varep}{2}\Big).
$$

Consider $I_1$
\begin{align*}
&\Big|\widetilde{R}_{\text{emp}}(\widehat{f}, \widehat{\beta})-\widetilde{R}_{\text{emp}}(\widehat{f}, \beta(\widehat f))\Big|\\
&=
\Big|
\frac{1}{n}\sum_{i=1}^{n}\Big(y_i^{\top}\theta^*-\left<\Phi(x_i)-\overline{\Phi}\,,\, \widehat f \,\right>_\Hcal\Big)^{2}-\frac{1}{n}\sum_{i=1}^{n}\Big(y_i^{\top}\theta^* - \overline{Y\theta^*}-\left<\Phi(x_i)-\overline{\Phi}\,,\, \widehat{f}\,\right>_\Hcal\Big)^{2}
\Big|\\
&=\Big|2\frac1{n}\sum_{i=1}^n\overline{Y\theta^*}\Big(y_i^{\top}\theta^* - \left<\Phi(x_i) - \overline{\Phi}, \widehat f \,\right>_\Hcal\Big) - \frac1{n}\sum_{i=1}^n(\overline{Y\theta^*})^2\Big|\\
&=\Big|(\overline{Y\theta^*})^2 - 2(\overline{Y\theta^*})\frac1{n}\sum_{i=1}^n\left<\Phi(x_i) - \overline{\Phi}, \widehat f \,\right>_\Hcal\Big|\\
&= |\overline{Y\theta^*}|^2.
\end{align*}

By Lemma~\ref{lem:ytheta}, there exists $C_1>0$ such that $\mathbb{P}(I_1>\varep/2)\leq 2\exp(-C_1 n \varepsilon)$ for all $\varep>0.$ By Theorem~\ref{thm:DifferenceRiskSupConverge}, there exists constants $C_2, C_3>0$ such that $\mathbb{P}(I_2>\varep/2)\leq C_2 \exp[-C_3 (n \varepsilon^2)/\{1+(\kappa\tau)^2\}].$ Combining the bounds for $I_1$ and $I_2$ gives
\begin{align*}
\mathbb{P}\Big(\widetilde{R}_{\text{emp}}(\widehat{f}, \widehat{\beta})-\widetilde{R}_{\text{emp}}(\widetilde{f}, \widetilde{\beta})>\varep\Big)
&\leq 2\exp(-C_1 n \varepsilon)+C_2 \exp\Big(-\frac{C_3 n \varepsilon^2}{1+(\kappa\tau)^2}\Big)\\
&\leq C_4\exp\Big(-\frac{C_5 n \varepsilon^2}{1+(\kappa\tau)^2}\Big)
\end{align*}
for some constants $C_i>0$.
This completes the proof of Theorem~\ref{thm:ModRiskDiffConverge}.
\end{proof}

\begin{proof}[Proof of Theorem \ref{thm:ModConsistency}]
Consider
\begin{align*}
    \widetilde{R}_{\text{emp}}(\widetilde{f}, \widetilde{\beta})-R(f^\ast, \beta^\ast) &= \widetilde{R}_{\text{emp}}(\widetilde{f}, \widetilde{\beta}) - \widetilde R_{emp}(f^*, \beta^*) + \widetilde R_{emp}(f^*, \beta^*) - R(f^*, \beta^*)\\
    &\leq \widetilde R_{emp}(f^*, \beta^*) - R(f^*, \beta^*),
\end{align*}
where the last inequality follows since $\widetilde{R}_{\text{emp}}(\widetilde{f}, \widetilde{\beta})\leq \widetilde R_{emp}(f^*, \beta^*)$ by the definition of $\widetilde f$, $\widetilde \beta$.

Let $z_{i}:=|y_i^{\top}\theta^*-\beta^\ast-\left<\Phi(x_i),f^\ast\right>_\Hcal|^2$, then $\widetilde{R}_{\text{emp}}(f^\ast, \beta^\ast)=n^{-1}\sum_{i=1}^{n} z_{i}$ is the average of i.i.d. random variables with $\mathbb{E}z_i=R(f^\ast, \beta^\ast)$ by definition of expected risk. Since $|z_i|\leq4(\|\theta^*\|_\infty+\kappa \tau)^{2}$, by Hoeffding's inequality
$$
\mathbb{P}(| \widetilde{R}_{\text{emp}}(f^*, \beta^*)-R(f^*, \beta^*)|>\varepsilon)=\mathbb{P}\Big(\Big|n^{-1}\sum_{i=1}^{n} (z_{i}-\mathbb{E}z_i)\Big|>\varepsilon \Big)\leq 2\exp \Big(-\frac{n \varepsilon^2}{16(\|\theta^*\|_\infty+\kappa \tau)^{4}} \Big).
$$
\end{proof}

\begin{proof}[Proof of Theorem~\ref{thm:DifferenceRiskSupConverge}]
By definition of $R_{\text{emp}}(f)$ and $\widetilde{R}_{\text{emp}}(f, \beta(f))$,
\begin{align*}
    R_{\text{emp}}(f)-\widetilde{R}_{\text{emp}}(f, \beta(f))&=\frac{1}{n}\sum_{i=1}^{n}|y_i^{\top}\widehat \theta-\left<\Phi(x_i)-\overline{\Phi},f\right>_\Hcal|^{2} - \frac{1}{n}\sum_{i=1}^{n}|y_i^{\top}\theta^*-\beta(f)-\left<\Phi(x_i), f\right>_\Hcal|^{2}\\
    &=\frac{1}{n}\sum_{i=1}^{n}|y_i^{\top}\widehat \theta-\left<\Phi(x_i)-\overline{\Phi},f\right>_\Hcal|^{2}-\frac{1}{n}\sum_{i=1}^{n}|y_i^{\top}\theta^*-\overline{Y\theta^*}-\left<\Phi(x_i)-\overline{\Phi}, f\right>_\Hcal|^{2}.
\end{align*}
%The difference $|R_{\text{emp}}(f)-\widetilde{R}_{\text{emp}}(f, \beta(f))|$ is bounded above by
%\[
%\bigg|\frac{1}{n}\sum_{i=1}^{n}|y_i\theta^*-\overline{y\theta^*}-\left<\Phi(x_i)-\overline{\Phi}, f\right>_\Hcal|^{2}-
%\frac{1}{n}\sum_{i=1}^{n}|y_i\theta-\left<\Phi(x_i)-\overline{\Phi},f\right>_\Hcal|^{2}\bigg|
%\]
Expanding the squares and cancelling equal terms yields
\begin{align*}
&R_{\text{emp}}(f)-\widetilde{R}_{\text{emp}}(f, \beta(f))\\
&\quad =
\frac1{n}\sum_{i=1}^n\Big\{(y_i^{\top}\widehat \theta)^{2} - (y_i^{\top}\theta^*)^{2}- 2y_{i}^{\top}(\widehat \theta-\theta^*)\left<\Phi(x_i)-\overline{\Phi},f\right>_\Hcal-2\overline{Y\theta^*}\left<\Phi(x_i)-\overline{\Phi},f\right>_\Hcal+2y_{i}^{\top}\theta^*\overline{Y\theta^*}
-(\overline{Y\theta^*})^{2}\Big\}\\
&\quad = \frac1{n}\sum_{i=1}^n\Big\{(y_i^{\top}\widehat \theta)^{2}  - (y_i^{\top}\theta^*)^{2}\Big\}- \frac1{n}\sum_{i=1}^n\Big\{2y_{i}^{\top}(\widehat \theta-\theta^*)\left<\Phi(x_i)-\overline{\Phi},f\right>_\Hcal\Big\} + (\overline{Y\theta^*})^2\\
&\quad = I_1 + I_2(f) + I_3,
\end{align*}
where $I_1$ and $I_3$ are independent of $f$. By the union bound and de Morgan's law,
$$
\mathbb{P}\Big(\sup_{f\in \Hcal_\tau}|R_{\text{emp}}(f)-\widetilde{R}_{\text{emp}}(f, \beta(f))|>\varepsilon\Big)\leq \mathbb{P}\Big(|I_1|>\frac{\varepsilon}{3}\Big)+
\mathbb{P}\Big(\sup_{f\in \Hcal_\tau}|I_2(f)|>\frac{\varepsilon}{3}\Big)+\mathbb{P}\Big(|I_3|>\frac{\varepsilon}{3}\Big).
$$
We bound each probability separately. Since $y_i\in \R^2$ is an indicator vector of class membership for sample $i$, using the definition of $\widehat \theta$ and $\theta^*$
\[
|I_1|=\bigg|\frac{1}{n}\sum\Big\{(y_i^{\top}\widehat\theta)^{2}-(y_i^{\top}\theta^*)^{2}\Big\}\big|\leq \max_i|(y_i^{\top}\widehat\theta)^{2}-(y_i^{\top}\theta^*)^{2}| = \max\Big(|n_1/n_2-\pi_1/\pi_2|,|n_2/n_1-\pi_2/\pi_1|\Big).
\]
By Lemma~\ref{lem:OptScoreSupNormBound}, there exist $C_1, C_2>0$ such that $\mathbb{P}(|I_1|>\varep/3)\leq C_1\exp(-C_2 n \varep^2)$.

By H{\'o}lder's and Cauchy-Schwarz inequalities
\begin{align*}
|I_2(f)|&=
\bigg|\frac{1}{n}\sum_{i=1}^{n}2y_i^{\top}(\widehat \theta-\theta^*)\left<\Phi(x_i)-\overline{\Phi},f\right>_\Hcal\bigg|\\
&\leq\frac{1}{n}\sum_{i=1}^{n} 2|y_i^{\top}(\widehat \theta-\theta^*)|\cdot|\left<\Phi(x_i)-\overline{\Phi},f\right>_\Hcal|\\
&\leq 2 \|\widehat \theta-\theta^*\|_{\infty}\max_i|\left<\Phi(x_i)-\overline{\Phi},f\right>_\Hcal|\\
&\leq 2 \max\Big(|\sqrt{n_1/n_2}-\sqrt{\pi_1/\pi_2}|,|\sqrt{n_2/n_1}-\sqrt{\pi_2/\pi_1}|\Big)\max_i\|\Phi(x_i)-\overline{\Phi}\|_{\mathcal{H}}\,\|f\|_{\mathcal{H}}\\
&\leq 4\max\Big(|\sqrt{n_1/n_2}-\sqrt{\pi_1/\pi_2}|,|\sqrt{n_2/n_1}-\sqrt{\pi_2/\pi_1}|\Big)  \kappa\tau,
\end{align*}
where we used Assumption~\ref{assump:KernBound} in the last inequality. Since the upper bound does not depend on $f$, the same bound holds for $\sup_{f\in\Hcal_{\tau}}|I_2(f)|$. Combining the bound with Lemma~\ref{lem:OptScoreSupNormBound} gives for some $C_3, C_4>0$
$$
\mathbb{P}\Big(\sup_{f\in\Hcal_{\tau}}|I_2(f)|>\varep\Big)\leq
\mathbb{P}\bigg(\max\Big(|\sqrt{n_1/n_2}-\sqrt{\pi_1/\pi_2}|,|\sqrt{n_2/n_1}-\sqrt{\pi_2/\pi_1}|>\frac{\varep}{4\kappa \tau}\Big) 
\leq C_3\exp(-C_4 \frac{n \varepsilon^2}{(\kappa\tau)^2}).
$$

% The Cauchy-Schwarz inequality gives
%\[
%\|\left<\Phi(x_i)-\overline{\Phi},f\right>_\Hcal\|_{\mathcal{H}}\leq %\|\Phi(x_i)-\overline{\Phi}\|_{\mathcal{H}}\,\|f\|_{\mathcal{H}}\leq 
%\frac{2\kappa}{\sqrt{\gamma}},
%\]
%where $\kappa$ is the kernel bound of Assumption \ref{assump:KernBound}. We have used the inequality $\|f\|_{\mathcal{H}}\leq \frac{1}{\sqrt{\gamma}}$ from Lemma \ref{lem:Bounds}. Applying this bound, we have

By Lemma~\ref{lem:ytheta}, there exists $C_5>0$ such that $\mathbb{P}(|I_3|>\varep/3)\leq 2 \exp(-C_5 n \varep)$.

Combining the bounds for $I_1$, $I_2$ and $I_3$ gives
\begin{align*}
\mathbb{P}\big(\sup_{f\in \Hcal_{\tau}}| R_{\text{emp}}(f)-\widetilde{R}_{\text{emp}}(f, \beta(f))|>\varep\Big)
&\leq C_1\exp(-C_2 n \varepsilon^2)+C_3\exp(-C_4 \frac{n \varepsilon^2}{(\kappa\tau)^2})+2 \exp(-C_5 n \varep)\\
&\leq C_6\exp\Big(-C_7 \frac{n \varepsilon^2}{1 + (\kappa\tau)^2}\Big)
\end{align*}
for some $C_6, C_7>0$. This completes the proof of Theorem~\ref{thm:DifferenceRiskSupConverge}.
\end{proof}

\section{Supplementary Lemmas}\label{sec:ProofOfLemmas}

\begin{lemma}\label{lem:lambdamax}
Consider minimizing
$f(w) = 2^{-1}w^{\top} Q w-\beta^Tw +2^{-1}\lambda\|w\|_{1}$ with respect to $w\in \R^p$ with $w_i\in[-1,1]$, where $Q$ is positive semi-definite and $\lambda\geq 0$. If $\lambda\geq 2\|\beta\|_\infty$, then the minimizing $w$ is the zero vector.
\end{lemma}

\begin{proof}
Consider $2^{-1}\lambda\|w\|_{1}-\beta^{\top}w=\sum_{i=1}^{p} (\lambda/2|w_i|-\beta_i w_i).$ If $\lambda \geq 2\|\beta\|_\infty$, this expression is non-negative for all $w\in \mathbb{R}^{p}$ and a minimum occurs at $w=0$. Since $Q$ is positive semi-definite, $w^{\top} \frac{1}{2} Q w$ is always non-negative with a minimum at $w=0$. It follows that for $\lambda \geq 2\|\beta\|_\infty$ the sum of these terms attains minimum at $w=0$.
\end{proof}

\begin{lemma}\label{lem:OpNormBound}
Let $M = [(C\mathbf{K}C)^2+n\gamma (C\mathbf{K}C)]^{-}C\mathbf{K}C$, then $\|M\|_{\text{op}}\leq (n\gamma)^{-1}$. 
\end{lemma}
\begin{proof}[Proof of Lemma~\ref{lem:OpNormBound}]
The kernel matrix $K$ is positive semi-definite since by the reproducing property for any $\alpha \in \R^n$ 
$$
\alpha^{\top}\mathbf{K}\alpha=\left<\sum_{i=1}^{n} \alpha_{i}\Phi(x_i)\,,\, \sum_{i=1}^{n}\alpha_{i}\Phi(x_i)\right>_\Hcal=\Big\|\sum_{i=1}^{n} \alpha_{i}\Phi(x_i)\Big\|^{2}_\Hcal\geq 0.
$$
It follows that $C\mathbf{K}C$ is also positive semi-definite. Let $\{\lambda_{i}\}_{i=1}^{k}$ be the set of non-zero eigenvalues of $C\mathbf{K}C$, then 
$\{\lambda_{i}/(\lambda_{i}^{2}+n\gamma \lambda_{i})\}_{i=1}^{k}$ are the non-zero eigenvalues of $M=[(C\mathbf{K}C)^2+n\gamma (C\mathbf{K}C)]^{-}C\mathbf{K}C$. The function $t\mapsto t/(t^2+n\gamma t)$ is bounded above by $(n\gamma)^{-1}$ for $t>0$, hence $\|M\|_{\text{op}}\leq (n\gamma)^{-1}$.
\end{proof}

\begin{lemma}\label{lem:Bounds} Let $\gamma >0$. The minimizer $\widehat f$ in~\eqref{eq:RegKernOptScore} satisfies $\|\widehat{f}\|_{\mathcal{H}}\leq 1/\sqrt{\gamma}$. Additionally, if Assumption~\ref{assump:KernBound} holds for $\kappa >0$, then $\|\widehat{f}\|_{\mathcal{H}}\leq 2\kappa/\gamma$.
\end{lemma}

\begin{proof}[Proof of Lemma \ref{lem:Bounds}]
Comparing the value of objective function in~\eqref{eq:RegKernOptScore} at $f = \widehat f$ with the value at $f=0$ gives
\begin{align*}
\gamma\|\widehat{f}\|_{\mathcal{H}}^{2}&\leq \frac{1}{n}\sum_{i=1}^{n}\Big|y_i^{\top}\widehat \theta-\left<\Phi(x_i)-\overline{\Phi},\widehat{f}\,\right>_\Hcal\Big|^{2}+\gamma\|\widehat{f}\|_{\mathcal{H}}^{2}\leq \frac{1}{n}\sum_{i=1}^{n}|y_i^{\top}\widehat \theta|^{2}=1.,
\end{align*}
where the last equality follows since $n^{-1}\widehat \theta Y^{\top}Y\widehat \theta = 1$.
It follows that $\|\widehat f\|_{\Hcal}\leq 1/\sqrt{\gamma}$. 

On the other hand, since $\widehat{f} = \sum_{i=1}^{n}\alpha_{i}(\Phi(x_i)-\overline{\Phi})$, by the triangle inequality and Assumption~\ref{assump:KernBound}
\begin{align*}
\|\widehat{f}\|_\mathcal{H}=
\Big\|\sum_{i=1}^{n}\alpha_{i}(\Phi(x_i)-\overline{\Phi})\Big\|_\mathcal{H}\leq \sum_{i=1}^{n}|\alpha_i|\|\Phi(x_i)-\overline{\Phi}\|_{\mathcal{H}}\leq \max_i \|\Phi(x_i)-\overline{\Phi}\|_{\mathcal{H}} \|\alpha\|_1
&\leq 2\kappa \|\alpha\|_{1}\leq 2\kappa\sqrt{n}\|\alpha\|_{2}.
\end{align*}
Since $\alpha = \{(C\mathbf{K}C)^{2}+\gamma n C\mathbf{K}C\}^{-}C\mathbf{K}C Y\widehat \theta$, applying Lemma \ref{lem:OpNormBound} and using $\|Y\widehat \theta\|_2 = \sqrt{\widehat \theta Y^{\top}Y\widehat \theta} = \sqrt{n}$ gives
\begin{align*}
\|\alpha\|_{2}\leq 
\| \{(C\mathbf{K}C)^{2}+\gamma n C\mathbf{K}C\}^{-}C\mathbf{K}C\|_{\text{op}}\| Y\widehat \theta\|_{2}
\leq\frac{\|Y\widehat \theta\|_{2}}{n\gamma} \leq \frac1{\sqrt{n}\gamma}.
\end{align*}
Combining the above two displays gives $\|\widehat{f}\|_{\mathcal{H}}\leq 2\kappa/\gamma$.
\end{proof}

\begin{lemma}\label{lem:EpsilonNetInSymmetrization}
Under Assumptions \ref{assump:ScoreBound} and \ref{assump:KernBound},
let $\{(x_i, y_i)\}_{i=1}^{n}$ and $\{(x_j, y_j)\}_{j=n+1}^{2n}$ be two independent copies of i.i.d. data, and let $T_x$ be the empirical measure on their union. Let $\widetilde{R}_{\text{emp}}(f,\beta)$ be the modified empirical risk on $\{(x_i, y_i)\}_{i=1}^{n}$, and $\widetilde{R}'_{\text{emp}}(f,\beta)$ on $\{(x_j, y_j)\}_{i=n+1}^{2n}$. Let $c=64(\|\theta^\ast\|_\infty+\tau \kappa)$, and let $\{f_1, \dots, f_M\}$ be the smallest $L^2(T_x)$ $\varepsilon/\sqrt{2}c$-net of $\mathcal{H}_\tau$, and let  $\{\beta_1, \dots, \beta_K\}$ be an $\varepsilon/c$-net of $I_\tau$. Then
\begin{align*}%\label{eq:EpslnNetSymmetrization}
\begin{split}
\mathbb{P}\bigg(\sup_{\substack{f\in H_\tau\\
\beta\in I_\tau}}\{\widetilde{R}_{\text{emp}}(f,\beta)-\widetilde{R}_{\text{emp}}'(f,\beta)\}>\frac{\varepsilon}{2} \bigg) 
\leq\mathbb{P}\bigg( 
\maximize_{\substack{f\in \{f_1, \dots, f_M\}\\ \beta\in \{\beta_1, \dots, \beta_K\}}}
\{\widetilde{R}_{\text{emp}}(f,\beta)-\widetilde{R}_{\text{emp}}'(f,\beta)\}>\frac{\varepsilon}{4}
\bigg).
\end{split}
\end{align*}
\end{lemma}

\begin{proof}[Proof of Lemma~\ref{lem:EpsilonNetInSymmetrization}]

%Let $\{(x_i,y_i)\}_{i=1}^{2n}$ be i.i.d. data. Let $\{f_1, \dots, f_M\}$ be an $L^2(T_x)$ $\varepsilon/\sqrt{2}c$-net of $\mathcal{H}_\tau$, and let $\{\beta_1, \dots,\beta_K\}$ be an $\varepsilon/c$ net of $I_\tau$, with $c=64(\|\theta^*\|_\infty+\tau\kappa)$. Let $\widetilde{R}_\text{emp}$ and $\widetilde{R}_\text{emp}'$ be the modified empirical risks on  $\{(x_i,y_i)\}_{i=1}^{n}$ and $\{(x_i,y_i)\}_{i=n+1}^{2n}$, respectively. 

Let $f\in \mathcal{H}_\tau$, $\beta\in I_\tau$ be such that $\widetilde R_{\text{emp}}(f,\beta)-\widetilde R'_{\text{emp}}(f,\beta)>\varepsilon/2$. There exists $f_j\in \{f_1, \dots, f_M\}$ and $\beta_\ell\in \{\beta_1, \dots, \beta_K\}$ such that $\|f_j-f\|_{L^2(T_x)}<\varepsilon/\sqrt{2}c$ and $|\beta-\beta_\ell|<\varepsilon/c$. Applying Lemma \ref{lem:EpslnNetRestriction} gives
\[
\sqrt{\frac{1}{n}\sum_{i=1}^{n}|f(x_i)-f_j(x_i)|^{2}}<\frac{\varepsilon}{c}\quad\text{and}\quad
\sqrt{\frac{1}{n}\sum_{i=n+1}^{2n}|f(x_i)-f_j(x_i)|^{2}}<\frac{\varepsilon}{c}.
\]
Applying Lemma \ref{lem:ClosenessOfEmpRisk} yields
\[
|\widetilde{R}_{\text{emp}}(f,\beta)-\widetilde{R}_{\text{emp}}(f_j, \beta_\ell)|<8\frac{\varepsilon}{c}(\|\theta^*\|_{\infty} + \kappa \tau)=
\frac{\varepsilon}{8},
\]
and similarly $|\widetilde{R}'_{\text{emp}}(f,\beta)-\widetilde{R}'_{\text{emp}}(f_j, \beta_\ell)|< \varepsilon/8$.
Therefore, $\widetilde R'_{\text{emp}}(f,\beta)-\widetilde R_{\text{emp}}(f,\beta)>\varepsilon/2$ for some $f\in \Hcal_{\tau}$, $\beta \in I_{\tau}$ implies $
\widetilde{R}'_{\text{emp}}(f_j, \beta_\ell)-\widetilde{R}_{\text{emp}}(f_j, \beta_\ell)>\varepsilon/4
$ for some $f_j$ and $\beta_\ell$. Therefore,
\[
\mathbb{P}\bigg(\sup_{f\in \mathcal{H}_\tau,\, \beta\in I_\tau}\{\widetilde R'_{\text{emp}}(f,\beta)-\widetilde R_{\text{emp}}(f,\beta)\}>\frac{\varepsilon}{2}\bigg)\leq 
\mathbb{P}\bigg( \maximize_{\substack{f\in \{f_1, \dots, f_M\}\\
\beta\in \{\beta_1, \dots, \beta_K\}}}\{\widetilde{R}'_{\text{emp}}(f_j, \beta_\ell)-\widetilde{R}_{\text{emp}}(f_j, \beta_\ell)\}>\frac{\varepsilon}{4}
\bigg).
\]
\end{proof}

%Lemma \ref{lem:CoveringNum} bounds the $L^2(T_x)$ $\varepsilon$-covering number of $\mathcal{H}_\tau. This bound is used in proving the following theorem. 

\begin{lemma}\label{lem:MaxBound}
Under Assumptions \ref{assump:ScoreBound}-\ref{assump:Separable}, let $\{f_1, \dots, f_M\}$ and $\{\beta_1, \dots, \beta_K\}$ be as in Lemma~\ref{lem:EpsilonNetInSymmetrization}. %, and let $4(\|\theta^*\|_\infty+\kappa\tau)^2$ be the bound on the empirical risk error terms. 
There exist a constant $C_1>0$ such that for all $\varep>0$,
\begin{align*}
\mathbb{P}\bigg( 
\maximize_{\substack{f\in \{f_1, \dots, f_M\}\\ \beta\in \{\beta_1, \dots, \beta_K\}}}
\{\widetilde{R}_{\text{emp}}(f,\beta)-\widetilde{R}_{\text{emp}}'(f,\beta)\}>\frac{\varepsilon}{4}
\bigg)\leq 
\mathcal{N}_{\varep} \,\exp\Big(-  \frac{n\varepsilon^2}{128(\|\theta^*\|_\infty+\kappa \tau)^{4}}\Big),
\end{align*}
where $\mathcal{N}_{\varep}=\{1+2(\|\theta^*\|_\infty+\kappa\tau)/\varep\} 
\exp( C_1\tau^2 \varepsilon^{-2})$.
\end{lemma}

\begin{proof}[Proof of Lemma \ref{lem:MaxBound}]
 Let $\sigma = \{\sigma_i\}_{i=1}^n$ be $i.i.d.$ Radamacher random variables,  $\mathbb{P}(\sigma_i=1)=\mathbb{P}(\sigma_i=-1)=1/2$. Let $$
 \widetilde{R}_{\text{emp}}^{\sigma}=
 \frac{1}{n}\sum_{i=1}^{n}\sigma_{i}|y_i^{\top}\theta^*-\beta-\left<\Phi(x_i),f\right>_\Hcal|^{2}, \quad \widetilde{R}_{\text{emp}}^{'\sigma}=
  \frac{1}{n}\sum_{i=n+1}^{2n}\sigma_{i}|y_i^{\top}\theta^*-\beta-\left<\Phi(x_i),f\right>_\Hcal|^{2}.
 $$
Since $(y_i, x_i)$ and $(y_{n+i}, x_{n+i})$ are independent, and have the same distribution, the distribution of $\xi_{i}:=(|y_{i}^\top \theta^*-\beta-\left<\Phi(x_{i}), f\right>_\Hcal|^{2}-
|y_{n+i}^\top\theta^*-\beta-\left<\Phi(x_{n+i}), f\right>_\Hcal|^{2})$ is the same as distribution of $\sigma_i\xi_i$. Let $Z=\{(x_i, y_i)\}_{i=1}^{2n}$, then
\[
\mathbb{P}_Z\bigg( 
\max_{\substack{f\in \{f_1, \dots, f_M\}\\ \beta\in \{\beta_1, \dots, \beta_K\}}}
\{\widetilde{R}_{\text{emp}}(f,\beta)-\widetilde{R}_{\text{emp}}'(f,\beta)\}>\frac{\varepsilon}{4}
\bigg)=\mathbb{P}_{Z,\sigma}\bigg(\max_{\substack{f\in \{f_1, \dots, f_M\}\\\beta\in \{\beta_1, \dots, \beta_K\}}}\{\widetilde{R}_{\text{emp}}^{\sigma}(f,\beta)-\widetilde{R}_{\text{emp}}^{'\sigma}(f,\beta)\}>\frac{\varepsilon}{4}\bigg).
\]
%Radamacher variables $\sigma_i$ are typically used within expectations in order to bound the Radamacher complexity of a class of functions, as in \cite[Chapter 26]{shalev2014understanding}. 
%\cite[Chapter 5]{bernhard_scholkopf_learning_2002} uses a similar procedure as we do but uses random permutations of the data rather than Radamacher variables. Our use of Radamacher variables allows for direct use of Hoeffding's inequality.
Let $\Acal_{m,k}$ be the event $\Acal_{m,k} = \{\widetilde{R}_{\text{emp}}^{\sigma}(f_m,\beta_k)-\widetilde{R}_{\text{emp}}^{'\sigma}(f_m,\beta_k) > \varepsilon/4\}$ for $m=1,\dots,M(Z)$; $k=1,\dots, K$; where $M(Z)$ emphasizes the dependence of $M$ on $Z$. Using properties of conditional expectation and union bound
\begin{align*}
\mathbb{P}_{Z,\sigma}\bigg(\max_{\substack{f\in \{f_1, \dots, f_M\}\\\beta\in \{\beta_1, \dots, \beta_K\}}}\{\widetilde{R}_{\text{emp}}^{\sigma}(f,\beta)-\widetilde{R}_{\text{emp}}^{'\sigma}(f,\beta)\}>\frac{\varepsilon}{4}\bigg) &= \mathbb{P}_{Z, \sigma}(\cup_{m=1}^{M(Z)}\cup_{k=1}^K \Acal_{m,k})\\
&=\E_{Z}\left\{\mathbb{P}_{\sigma}(\cup_{m=1}^{M(Z)}\cup_{k=1}^K \Acal_{m,k}|Z)\right\}\\
&\leq \E_Z\left\{M(Z)K\mathbb{P}_{\sigma}(\Acal_{m,k}|Z)\right\}.
\end{align*}

For fixed $f_m$, $\beta_k$ and conditionally on $Z$, the terms
$\psi_{i}:=\sigma_{i}(|y_{i}^\top \theta^*-\beta_{k}-\left<\Phi(x_{i}), f_{m}\right>_\Hcal|^{2}-
|y_{n+i}^\top\theta^*-\beta_{k}-\left<\Phi(x_{n+i}), f_{m}\right>_\Hcal|^{2})$, $i=1, \dots, n$, 
are independent, mean-zero random variables with $|\psi_i|\leq  4(\|\theta^*\|_\infty+\kappa \tau)^{2}$.
Applying Hoeffding's inequality gives
$$
\mathbb{P}_{\sigma}(\Acal_{m,k}|Z) = \mathbb{P}_{\sigma}\Big(\frac1{n}\sum_{i=1}^n \psi_i > \varepsilon / 4\,\Big|\, Z\Big) \leq \exp\Big(-  \frac{n\varepsilon^2}{128(\|\theta^*\|_\infty+\kappa \tau)^{4}}\Big).%,
$$

%where $C>0$ is a constant depending on $\|\theta^*\|_{\infty}$, $\kappa$ and $\tau$.
On the other hand, since $I_{\tau}$ is a one-dimensional sphere of radius $\|\theta^*\| + \kappa \tau$, $K$ is independent of the data and $K \leq 1 + 2(\|\theta^*\|_\infty+\kappa \tau)/\varepsilon$. Combining this with the above two displays gives 
\begin{align*}
&\mathbb{P}_{Z,\sigma}\bigg(\max_{\substack{f\in \{f_1, \dots, f_M\}\\\beta\in \{\beta_1, \dots, \beta_K\}}}\{\widetilde{R}_{\text{emp}}^{\sigma}(f,\beta)-\widetilde{R}_{\text{emp}}^{'\sigma}(f,\beta)\}>\frac{\varepsilon}{4}\bigg)\\ &\leq \{1 + 2(\|\theta^*\|_\infty+\kappa \tau)/\varepsilon\}\,\E_Z\{M(Z)\}\exp\Big(-  \frac{n\varepsilon^2}{128(\|\theta^*\|_\infty+\kappa \tau)^{4}}\Big).
\end{align*}

Recall that $\{f_1, \dots, f_M\}$ is the smallest $L^2(T_x)$ $\varepsilon/\sqrt{2}c$-net of $\mathcal{H}_\tau$, with $c=64(\|\theta^\ast\|_\infty+\tau \kappa).$ By Lemma~\ref{lem:CoveringNum}
\begin{equation}\label{eq:UpperBoundCoveringNumber}
\mathbb{E}_Z\{M(Z)\}\leq \sup_{Z=\{(x_i, y_i)\}_{i=1}^{2n}} M(Z)\leq\exp\bigg(\frac{C_1\tau^2}{ \varepsilon^2}\bigg)
\end{equation}
for some constant $C_1>0$. Setting $\mathcal{N}_{\varep}=\{1+2(\|\theta^*\|_\infty+\kappa\tau)/\varep\} 
\exp( C_1\tau^2 \varepsilon^{-2})$ completes the proof of Lemma~\ref{lem:MaxBound}.
\end{proof}

\begin{lemma}\label{lem:OptScoreSupNormBound}
Under Assumption \ref{assump:ScoreBound} there exist constants $C_1, C_2>0$ such that for all $\varep>0$,
$$
\mathbb{P}\bigg(\max\Big(|n_1/n_2-\pi_1/\pi_2|,|n_2/n_1-\pi_2/\pi_1|\Big)>\varep\bigg)\leq C_1 \exp\Big(- C_2 n \varepsilon^2 \Big),
$$
$$
\mathbb{P}\bigg(\max\Big(|\sqrt{n_1/n_2}-\sqrt{\pi_1/\pi_2}|,|\sqrt{n_2/n_1}-\sqrt{\pi_2/\pi_1}|\Big)>\varep\bigg)\leq C_1 \exp\Big(- C_2 n \varepsilon^2 \Big).
$$

%$$
%|n_1/n_2-\pi_1/\pi_2|\leq C_1\bigg\{\log(\eta^{-1})/n\bigg\}^{1/2}, \quad %|\sqrt{n_1/n_2}-\sqrt{\pi_1/\pi_2}|\leq C_2\bigg\{\log(\eta^{-1})/n\bigg\}^{1/2}.
%$$
\end{lemma}
\begin{proof}[Proof of Lemma~\ref{lem:OptScoreSupNormBound}]
We provide the proof for $n_1/n_2$, the proof for $n_2/n_1$ is analogous. The first inequality is equivalent to Lemma~1 in~\cite{gaynanova2017sparse}. For the second inequality, by Taylor expansion of the square root function
%for a fixed $x>0$
%$$
%\sqrt{x+h}-\sqrt{x}=\frac1{2\sqrt{x}}h+o(h)\quad\mbox{as}\quad h\to 0.
%$$
centered at $\pi_{1}/\pi_{2}$
\begin{align*}
    \sqrt{n_1/n_2} - \sqrt{\pi_1/\pi_2} = 2^{-1}\sqrt{\pi_2/\pi_1}(n_1/n_2-\pi_1/\pi_2) + o(n_1/n_2-\pi_1/\pi_2).
\end{align*}
Since $|n_1/n_2-\pi_1/\pi_2|= O_p(n^{-1/2})$ by the first inequality, it follows that there exist a constant $C_3>0$ such that $|\sqrt{n_1/n_2} - \sqrt{\pi_1/\pi_2}|\leq C_2\{\log(\eta^{-1})/n\}^{1/2}$ with probability at least $1-\eta$. Setting $\varep=C_3\{\log(\eta^{-1})/n\}^{1/2}$ and solving for $\eta$ completes the proof. 
\end{proof}

\begin{lemma}\label{lem:ytheta} Let Assumption \ref{assump:ScoreBound} be true. For all $\varep>0$, we have $\mathbb{P}\big((\overline{Y\theta^*})^2>\varep \big)\leq  2 \exp(- n \varepsilon/\|\theta^*\|_\infty)$.
\end{lemma}
\begin{proof}[Proof of Lemma~\ref{lem:ytheta}]
Let $z_i=y_i^{\top}\theta^*$, then $z_i$ are independent,
$$
\E(z_i) = \E(y_i)^{\top}\theta^* =
\pi_1\sqrt{\frac{\pi_2}{\pi_1}} -\pi_2\sqrt{\frac{\pi_1}{\pi_2}}=
\sqrt{\pi_1 \pi_2}-\sqrt{\pi_1 \pi_2}=0
$$
and
$$
(\overline{Y\theta^*})^2 = (n^{-1}\sum_{i=1}^ny_i^{\top}\theta^*)^2=(n^{-1}\sum_{i=1}^nz_i)^2.
$$
Since $|z_i|\leq \|\theta^*\|_{\infty}=\sqrt{\pi_{\max}/\pi_{\min}}$, by Hoeffding's inequality for $\varepsilon >0$ 
\begin{align*}
\mathbb{P}\Big(\Big|n^{-1}\sum_{i=1}^nz_i\Big|^2>\varepsilon\Big) = \mathbb{P}\Big(\Big|n^{-1}\sum_{i=1}^nz_i\Big|>\sqrt{\varepsilon}\Big)\leq 2\exp(-n \varepsilon/\|\theta^\ast\|_\infty).
\end{align*}
\end{proof}

\begin{lemma}\label{lem:ClosenessOfEmpRisk}
Let Assumptions \ref{assump:ScoreBound} and \ref{assump:KernBound} be true, and suppose that $\{f_1, \dots, f_{M}\}$ is an $L^2 (T_{x})$ $\varepsilon$-net of $\mathcal{H}_\tau$ and that $\{\beta_1, \dots, \beta_K\}$ be an $\varepsilon$-net of $I_\tau$. 
Then for any admissible $f$  and $\beta$, let $f_j$ and $\beta_\ell$ be members of the $\varepsilon$-nets so that $\|f-f_j\|_{L^2(T_x)}<\varepsilon$ and 
$|\beta-\beta_\ell|<\varepsilon$. Then
\begin{equation}\label{eq:DiffLoss}
\left|\widetilde R_{emp}(f, \beta) - \widetilde R_{emp}(f_j, \beta_l)\right| \leq 8\varepsilon \Big(\|\theta^\ast\|_\infty +\kappa\tau \Big).
\end{equation}
%\begin{equation}\label{eq:DiffLoss}
%\bigg|\frac{1}{n}\sum_{i=1}^{n}|y_i^{\top}\theta^*-\beta-\left<\Phi(x_i),f\right>_\Hcal|^{2}-\frac{1}{n}\sum_{i=1}^{n}|y_i^{\top}\theta^*-\beta_\ell-\left<\Phi(x_i),f_j\right>_\Hcal|^{2}\bigg|<
%4\varepsilon \Big(\|\theta^\ast\|_\infty +(1+\kappa) \tau\Big).
%\end{equation}
\end{lemma}

\begin{proof}[Proof of Lemma \ref{lem:ClosenessOfEmpRisk}]
By the reproducing property of $\Hcal$, $\left<\Phi(x_i), f\right>_\Hcal= f(x_i)$, and
\begin{align*}
&\Big|\widetilde R_{emp}(f, \beta) - \widetilde R_{emp}(f_j, \beta_l)\Big|\\
&=\Big|\frac{1}{n}\sum_{i=1}^{n}|y_i^{\top}\theta^*-\beta-\left<\Phi(x_i),f\right>_\Hcal|^{2}-\frac{1}{n}\sum_{i=1}^{n}|y_i^{\top}\theta^*-\beta_\ell-\left<\Phi(x_i),f_j\right>_\Hcal|^{2}\Big|\\
&=\Big|\frac{1}{n}\sum_{i=1}^{n}|y_i^{\top}\theta^*-\beta-f(x_i)|^{2}-\frac{1}{n}\sum_{i=1}^{n}|y_i^{\top}\theta^*-\beta_\ell-f_j(x_l)|^{2}\Big|\\
&=\Big|-2\frac{1}{n}\sum_{i=1}^{n}y_i^{\top}\theta^*\{\beta+ f(x_i)-\beta_\ell-f_j(x_i)\}+\frac{1}{n}\sum_{i=1}^n[\{\beta+f(x_i)\}^2-\{\beta_\ell+f_j(x_i)\}^{2}]\Big| \\
&\leq \underbrace{2\|\theta^*\|_{\infty}\Big|\beta-\beta_l+\frac1{n}\sum_{i=1}^{n}\{f(x_i)-f_j(x_i)\}\Big|}_{I_1}+\underbrace{\Big|\frac{1}{n}\sum_{i=1}^n[\{\beta+f(x_i)\}^2-\{\beta_\ell+f_j(x_i)\}^{2}] \Big|}_{I_2}.
\end{align*}

Consider 
\begin{align*}
I_1 = 2\|\theta^*\|_{\infty}\bigg|\beta-\beta_l+\frac1{n}\sum_{i=1}^{n}\{f(x_i)-f_j(x_i)\}\bigg|&\leq 2\|\theta^*\|_{\infty}\left\{|\beta -\beta_l|+\frac{1}{n}\sum_{i=1}^{n}|f(x_i)-f_j(x_i)|\right\}\\
&\leq 2\|\theta^*\|_{\infty}\left\{\varepsilon + \Big[\frac1{n}\sum_{i=1}^n|f(x_i)-f_j(x_i)|^2\Big]^{1/2}\right\}\\
&\leq 4\|\theta^*\|_{\infty}\varepsilon,
\end{align*}
where we used $n^{-1}\sum_{i=1}^{n}[|f(x_i)-f_j(x_i)|^2]^{1/2} \leq [n^{-1}\sum_{i=1}^n|f(x_i)-f_j(x_i)|^2]^{1/2}$ due to Jensen's inequality, and that $\|f-f_j\|_{L^2(T_x)}<\varepsilon$ and 
$|\beta-\beta_\ell|<\varepsilon$.

Consider $I_2$. Using $a^2-b^2=(a+b)(a-b)$, the Cauchy-Schwarz inequalty, and Jensen's inequality,
\begin{align*}
I_2 &= \frac{1}{n}\bigg|\sum_{i=1}^{n}\{\beta+f(x_i)+\beta_\ell+f_j(x_i)\}\{\beta-\beta_\ell+f(x_i)-f_j(x_i)\}\bigg|\\
&\leq 2(\sup_{\beta\in I_\tau}|\beta|+ \sup_{x,f\in \mathcal{H_\tau}}|f(x)|)\frac{1}{n}\sum_{i=1}^{n}(|\beta-\beta_j|+|f(x_i)-f_j(x_i)|)\\
&\leq 2(\|\theta^\ast\|_\infty +\kappa \tau+ \sup_{x,f\in \mathcal{H_\tau}}|\left< \Phi(x),f \right>_\Hcal|)(\varepsilon+\frac{1}{n}\sum_{i=1}^{n}|f(x_i)-f_j(x_i)|)\\
&\leq 2\bigg( 
\|\theta^\ast\|_\infty +\kappa \tau+\kappa \tau
\bigg)\bigg(\varepsilon+\sqrt{\frac{1}{n}\sum_{i=1}^{n}|f(x_i)-f_j(x_i)|^{2}}\bigg)\\
&=4\varepsilon \Big(\|\theta^\ast\|_\infty +2\kappa \tau \Big).
\end{align*}
Combining the bounds for $I_1$ and $I_2$ completes the proof of Lemma~\ref{lem:ClosenessOfEmpRisk}.
\end{proof}

\begin{lemma}\label{lem:EpslnNetRestriction}
Let $\{(x_i, y_i)\}_{i=1}^{2n}$ be the data, and consider an $L^{2}(T_x)$ $\varepsilon$-net $\{f_1, \dots, f_M\}$ of $\mathcal{H}_\tau$. Then $\{f_1, \dots, f_M\}$ is an $\sqrt{2}\varepsilon$-net with respect to the empirical measure on half of the data $\{(x_i, y_i)\}_{i=1}^{n}$.
\end{lemma}

\begin{proof}[Proof of Lemma \ref{lem:EpslnNetRestriction}] Since $\{f_1, \dots, f_M\}$ is $\varepsilon$-net with respect to $\{(x_i,y_i)\}_{i=1}^{2n}$, for any $f\in \Hcal_{\tau}$, there exists $f_j$ such that
\[
\sqrt{\frac{1}{2n}\sum_{i=1}^{2n}|f(x_i) - f_j(x_i)|^{2}}<\varepsilon. 
\]
If $\frac{1}{2n}\sum_{i=1}^{2n}|f(x_i) - f_j(x_i)|^{2}=0$, then $\frac{1}{n}\sum_{i=1}^{n}|f(x_i) - f_j(x_i)|^{2}=0$. Otherwise
\begin{align*}
\sqrt{\frac{1}{n}\sum_{i=1}^{n}|f(x_i) - f_j(x_i)|^{2}}&=
\sqrt{\frac{2n}{2n}\frac{1}{n}\sum_{i=1}^{n}|f(x_i) - f_j(x_i)|^{2}
\frac{\sum_{i=1}^{2n}|f(x_i) - f_j(x_i)|^{2}}{\sum_{i=1}^{2n}|f(x_i)- f_j(x_i)|^2}}\\
&=\sqrt{\frac{2n}{n} \frac{\sum_{i=1}^{n}|f(x_i)- f_j(x_i)|^{2}}{\sum_{i=1}^{2n}|f(x_i)- f_j(x_i)|^{2}}}\,
\sqrt{\frac{1}{2n}\sum_{i=1}^{2n}|f(x_i)- f_j(x_i)|^2}< \sqrt{2}\varepsilon,
\end{align*}
hence $\{f_1, \dots, f_M\}$ is $\sqrt{2}\varepsilon$-net with respect to $\{(x_i,y_i)\}_{i=1}^{n}$.
\end{proof}

\begin{lemma}[Theorem 2.1 of \cite{steinwart_fast_2007}]\label{lem:CoveringNum}
Let Assumption \ref{assump:Separable} be true, and
Let $M(Z)$ be the size of an $L^2(T_x)$ $\varepsilon$-covering number of $\Hcal_\tau$ with data $Z=\{(x_i, y_i)\}_{i=1}^{n}$.
There exists a $C>0$ independent of $n$, such that
\begin{equation}
\sup_{Z=\{(x_i, y_i)\}_{i=1}^{n}}M(Z)\leq \exp\bigg( \frac{ C\tau^2}{\varepsilon^2}\bigg).
\end{equation}
\end{lemma}
\begin{remark}
\cite{zhang2016quantile} notes that ``Theorem 2.1 of \cite{steinwart_fast_2007} considered
only the Gaussian RKHS, however the proof of the entropy bound for p = 2 in their notation
only requires that the RKHS is separable.'' It is this case which is presented in Lemma \ref{lem:CoveringNum}.
\end{remark}

\end{document}